\documentclass[preprint,12pt]{elsarticle}

\usepackage{amssymb}
\usepackage{savesym}
\savesymbol{Bbbk}
\savesymbol{openbox}

\usepackage{xcolor}
\usepackage{url}
\usepackage[utf8]{inputenc}
\usepackage{graphicx}
\usepackage{amsmath}
\usepackage{amssymb}
\usepackage{amsthm}
\usepackage{MnSymbol}
\usepackage{booktabs}
\usepackage{algorithm}
\usepackage{algorithmic}
\usepackage{caption}
\usepackage{subcaption}
\usepackage{multirow}
\restoresymbol{NEW}{Bbbk}
\restoresymbol{NEW}{openbox}
\usepackage{tcolorbox}

\newcommand{\AR}[1]{\textcolor{black}{#1}}
\newcommand{\JJ}[1]{\textcolor{black}{#1}}

\def\resp{respectively}

\def\Models{\ensuremath{\mathcal{M}}}
\def\model{\ensuremath{\mathcal{M}}}
\def\CFXs{\ensuremath{\mathcal{C}}}
\def\Classes{\ensuremath{\mathcal{L}}}
\def\Props{\ensuremath{\mathcal{P}}}
\def\inn{\ensuremath{\mathbf{x}}}

\def\ProblemFull{\textbf{Recourse-Aware Ensembling}}
\def\Problem{RAE}
\def\AggModels{\ensuremath{\mathcal{M}^{n}}}
\def\method{naive ensembling}

\def\Args{\ensuremath{\mathcal{X}}}

\def\Atts{\ensuremath{\mathcal{A}}}
\def\Supps{\ensuremath{\mathcal{S}}}

\def\arg{\ensuremath{\alpha}}
\def\arga{\ensuremath{\arg_1}}
\def\argb{\ensuremath{\arg_2}}
\def\argc{\ensuremath{\arg_3}}

\def\model{\ensuremath{M}}

\def\cfx{\ensuremath{\mathbf{c}}}

\def\prop{\ensuremath{\pi}}

\def\ppreceq{\preceq_\Props}
\def\psucc{\succ_\Props}

\def\psimeq{\simeq_\Props}

\def\mpreceq{\preceq_\Models}
\def\msucc{\succ_\Models}
\def\msucceq{\succeq_\Models}
\def\msimeq{\simeq_\Models}

\def\class{\ensuremath{\ell}}

\def\rel{\ensuremath{r}}

\newtheorem{definition}{Definition}
\newtheorem{theorem}{Theorem}
\newtheorem{lemma}{Lemma}

\newtheorem{example}{Example}
\newtheorem{proposition}{Proposition}

\newcommand\myprob[3]{%
  \begin{tcolorbox}
  \vspace{-0.2cm}
  {\bfseries Problem: #1}\\
  {\bfseries Input}: #2\\
  {\bfseries Output}: #3\par
  \vspace{-0.2cm}
  \end{tcolorbox}
}

\journal{}

\begin{document}

\begin{frontmatter}



\title{Argumentative Ensembling for Robust Recourse under Model Multiplicity}


\author[a]{Junqi Jiang\corref{cor1}}
\ead{junqi.jiang@imperial.ac.uk}
\author[a,b]{Antonio Rago\corref{cor1}}
\ead{a.rago@imperial.ac.uk}
\author[a]{Francesco Leofante}
\ead{f.leofante@imperial.ac.uk}
\author[a]{Francesca Toni}
\ead{f.toni@imperial.ac.uk}


\cortext[cor1]{These authors contributed equally.}
\affiliation[a]{organization={Department of Computing, Imperial College London},
            UK}
\affiliation[b]{organization={Department of Informatics, King's College London},
            UK}

\begin{abstract}
\JJ{In machine learning, it is common to obtain multiple equally performing models for the same prediction task, e.g., \AR{when} training neural networks with different random seeds. Model multiplicity (MM) \AR{is the situation which arises when} these competing models differ in 
their predictions for the same input, for which ensembling is often employed \AR{to determine an aggregation of the outputs}. 
\AR{Providing recourse recommendations via counterfactual explanations (CEs) under MM thus becomes complex, since the CE may not be valid across all models, i.e., the CEs are not robust under MM.} In this work, we formalise the problem of \AR{providing} recourse under MM, which we name \emph{recourse-aware ensembling (RAE)}. We 
\AR{propose the idea} that under MM, CEs for each individual model should be considered alongside their predictions so that the aggregated prediction and recourse are decided in 
\AR{tandem}. Centred around this intuition, we introduce six desirable properties for solution\AR{s} to this problem. For solving RAE, we first extend existing ensembling methods, and show that they fall short in terms of property satisfaction. Then, we propose a novel \emph{argumentative ensembling} method 
\AR{which guarantees the robustness of CEs under MM}. Specifically, our method leverages computational argumentation to explicitly represent the conflicts between models \AR{and counterfactuals} regarding prediction results and CE validity\AR{. It then uses argumentation semantics} 
\AR{to resolve} the conflicts and obtain the final solution\AR{, in a manner which is parametric to the chosen semantics}. \AR{O}ur method \AR{also} allows \AR{for the specification of} preferences over the models under MM\AR{, allowing further customisation of the ensemble}. In a comprehensive theoretical analysis, we characterise the behaviour\AR{, focusing on the aforementioned properties,} of argumentative ensembling with four different argumentation semantics. We \AR{then} empirically demonstrate\AR{, across 3 datasets,} the effectiveness of our approach in satisfying desirable properties with eight instantiations of our method with different semantics and model preferences.}
\end{abstract}




\end{frontmatter}

\section{Introduction}
\label{sec:intro}

\JJ{The phenomenon of Model Multiplicity (MM) \AR{occurs when multiple, equally performing models give conflicting predictions} for the same machine learning (ML) task \AR{\cite{Black_22}. These models may be obtained, e.g., from different random seeds, and may,}
e.g., model architectures, model types \AR{or} high-level properties like fairness and robustness.
This is also known as predictive multiplicity \cite{Marx_20} or the Rashomon effect \cite{breiman2001statistical}. MM has gained increasing attention when considering the trustworthiness of ML models, especially in scenarios where their predictions can substantially influence humans, like healthcare or finance \cite{ganesh2025curious}.} Consider the commonly used scenario of a loan application (Figure \ref{fig:page1_example}), where an individual modelled by input $\inn$ with features \textit{unemployed} status, \textit{33} years of age and \textit{low} credit rating applies for a loan. Assume the bank has trained a set of ML models $\Models=\{\model_1, \model_2, \model_3\}$ to predict whether the loan should be granted or not. Even though each $\model_i$ may exhibit good performance overall, their internal differences may lead to conflicts, e.g., if $\model_1(\inn) = \model_2(\inn) = 0$ (i.e., reject), while $\model_3(\inn) =1 $ (i.e., accept).

\begin{figure}[h]
    \centering
    \includegraphics[width=0.8\textwidth]{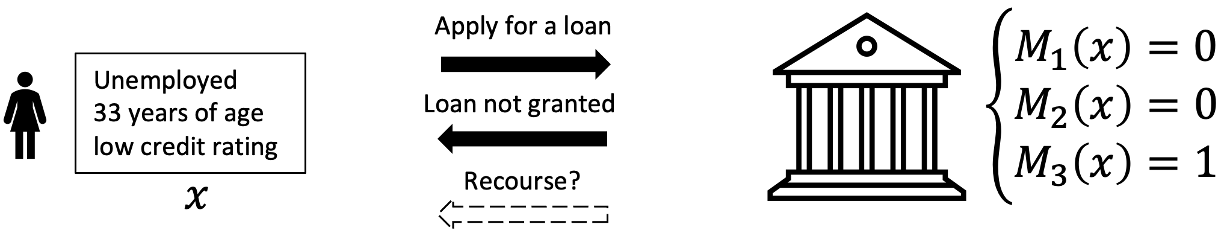}
	\caption{An example where a bank employs three ML models which reject a loan application by naive ensembling. Existing research ignores the need for recourse.
    \label{fig:page1_example}}
\end{figure}

\JJ{Recent research advocates taking into account the predictions from models under MM for inputs and reporting newly proposed metrics for evaluating MM severity, as an additional dimension of uncertainty that relevant stakeholders should be aware of \cite{semenova2022existence,hsu2022rashomon,watson2023predictive}. Additionally,} ensembling techniques are commonly used to deal with MM scenarios \cite{Black_22,black2022selective}. A standard such technique is \textit{\method}~\cite{black2022selective}, where the predictions of several models are aggregated to produce a single outcome that reflects the opinion of {a} majority of models. For instance, \method{} applied to our running example would result in a rejection, as a majority of the models agree that the {loan should not be granted} (Figure \ref{fig:page1_example}). \JJ{These efforts augment the standard ML practice of selecting the single most accurate model on a held-out test set, promoting better-informed predictions. However,} their application to consequential decision-making tasks raises some important challenges.

Indeed, 
these methods tend to ignore
the need to provide avenues for recourse to users negatively impacted by the models’ outputs, which the ML literature typically achieves via the provision of \emph{counterfactual explanations} (CEs) for the predictions (\JJ{see \cite{guidotti2022counterfactual,DBLP:journals/csur/KarimiBSV23,DBLP:conf/ijcai/Jiangsurvey24} for recent overviews}). \AR{This is likely because d}ealing with MM while also taking CEs into account is non-trivial. \JJ{Traditionally, CE generation methods assume access to a fixed model prediction, which is not directly available given the potentially conflicting predictions under MM. One apparent way for extending existing ensembling methods to account for recourse is to first retrieve an aggregated prediction, then compute CEs for this prediction. However,} 
CEs for
single models typically fail to produce recourse recommendations
that are valid across \JJ{the ensembled models \cite{pawelczyk2022uai}, and those that are designed to be robustly valid will be more difficult to achieve for the end user \cite{LeofanteBR23}. Such naively extended ensembling also ignores the rich information carried by the CEs about each model's local decision boundaries, which could further reflect similarities of subsets of models under MM CEs. Instead, as we show later in this work, determining the aggregated prediction and recourse in \AR{tandem} effectively addresses these limitations.}

Another challenge is that naive ensembling ignores other meta-evaluation aspects of models, like fairness, robustness, and interpretability, \AR{ignoring the fact that} models under MM can demonstrate substantial differences in these regards \cite{coston2021characterizing,rudin2019stop,d2022underspecification} 
 and users may have strong preferences for some of these aspects, e.g., they may prefer model fairness to accuracy \cite{Black_22}.

In this paper, we frame the recourse problem under MM formally and propose different approaches to accommodate recourse as well as user preferences \JJ{over models}. After covering related work (Section \ref{sec:related}) and the necessary preliminaries (Section \ref{sec:preliminaries}), we make the following contributions. 

\begin{itemize}
    \item \JJ{We propose the first formalisation of the problem, which we name \textit{recourse-aware ensembling} (RAE), along with six desirable properties which the solutions to RAE should satisfy (Section \ref{ssec:desirable_properties}).}
    \item \JJ{We present two natural extensions of naive ensembling to accommodate CEs for solving RAE, and theoretically show that they violate some important desirable properties (Section \ref{ssec:naive}).}
    \item \JJ{We then propose \textit{argumentative ensembling} \AR{(}Section \ref{sec:main}\AR{)}, a novel technique rooted in computational argumentation (see \cite{AImagazine17,handbook} for overviews). Specifically, we model the RAE problem into Bipolar Argumentation Frameworks (BAF), then apply four argumentation semantics \cite{Cayrol_05} to obtain the solution. Through extensive theoretical analysis, we show that it is able to solve RAE effectively while also naturally incorporating user preferences.}
    \item \JJ{We empirically evaluate eight instantiations of argumentative ensembling against the extended naive ensembling baselines \AR{(}Section \ref{sec:evaluation}\AR{)}. We show that our framework is effective in satisfying the desirable properties and addressing the model preferences without compromising the task performance.}
    \item \JJ{Finally, we conclude with discussion on limitations and future works \AR{(}Section \ref{sec:conclusion}\AR{)}.}
\end{itemize}

\JJ{This paper builds upon our previous work \cite{DBLP:conf/atal/JiangL0T24} with \AR{a number of significant} extensions. In Section \ref{sec:main}, we provide a more generalised definition of argumentative ensembling, parametrising the semantics that can be used in the BAF. 
We \AR{also} instantiate argumentative ensembling with four semantics (Section\AR{s} \ref{ssec:d_extension_theoretical} and \ref{ssec:s_extension_theoretical}), as opposed to one (s-preferred semantics) in the original work. We \AR{then} perform substantial new theoretical analyses throughout Section \ref{sec:main} for each instantiation of our method against the desirable properties we propose. In Section \ref{ssec:aaf}, we additionally discuss theoretical links between our method which uses bipolar argumentation and \AR{introduce} a \AR{novel,} alternative method \AR{which uses} abstract argumentation \cite{Dung_95}. Together with newly added illustrative examples throughout Section \ref{sec:main}, we present a more comprehensive demonstration of argumentative ensembling. 
\AR{Further, w}e restructured the evaluation metrics in Section \ref{sec:evaluation}, and separately analysed the effectiveness of our approach in terms of model preferences and property satisfaction, with improved visualisations. We \AR{then} added a new analysis on scalability with more experiments \AR{in} Section \ref{ssec:scalability}. Finally, we refined much of the paper with substantial details, better contextualising the research.} The implementation is publicly available at \url{https://github.com/junqi-jiang/argumentative_ensembling}. 



\section{Related Work}
\label{sec:related}

\AR{The intersection between explainable AI and solutions to the MM problem has not been given much attention to date. In this section we review work in this area after a brief overview of relevant work in MM at large. We also include work in computational argumentation, the basis for our proposed methodological solution
to recourse under MM.}

{\bf Model Multiplicity.} MM has been shown to affect several 
dimensions of trustworthy ML. In particular, among equally accurate models, there could 
be 
different fairness characteristics \cite{coston2021characterizing,wick2019unlocking,dutta2020there,rodolfa2021empirical}, levels of 
interpretability \cite{chen2018interpretable,rudin2019stop,semenova2022existence}, model robustness evaluations \cite{d2022underspecification} and even inconsistent explanations~\cite{dong2019variable,fisher2019all,mehrer2020individual,blackconsistent,ley2023consistent,marx2023but,LeofanteBR23}. 

Recent attempts have been made to 
address the MM problem. \citet{Black_22} suggested candidate models should be evaluated across additional dimensions other than accuracy (e.g., robustness or fairness evaluation thresholds). They provided a potential solution based on applying meta-rules to filter out undesirable models, then us
{ing} ensemble methods to aggregate them, or randomly select{ing} one of them. Extending these ideas, the selective ensembling method of \cite{black2022selective} embeds statistical testing into the ensembling process, such that when the numbers of candidate models predicting an input to the top two classes are close or equal (which could happen under naive ensembling strategies like majority voting), an abstention signal can be flagged for relevant stakeholders. 
\JJ{\citet{xin2022exploring} \AR{and} \citet{DBLP:conf/nips/Zhong00SR23} looked at efficiently enumerating all models obtainable under MM for decision trees and generalised additive models, respectively. Similarly, \citet{hsudropout,DBLP:conf/nips/HsuBSLC24} investigated obtaining a representative set of models under MM for neural networks and gradient-boosted trees.} \citet{roth2023reconciling} instead proposed a model reconciling procedure to resolve 
conflicts between 
disagreeing models. Meanwhile, a number of works \cite{Marx_20,semenova2022existence,hsu2022rashomon,watson2023predictive,DBLP:journals/corr/cavus24} propose metrics to quantify the severity of MM, characterising its impact in prediction tasks. \JJ{More recently, researchers also discuss MM problems in large language models and propose ways to obtain more consistent responses \AR{therefrom} \cite{hamman2024quantifying,DBLP:conf/atal/PotykaZHKS24,DBLP:conf/emnlp/ZhuPNXHKS24}.}



{\bf Counterfactual Explanations and MM.} 
A CE for a prediction of an input by an ML model is typically defined as another data point that minimally modifies the input such that the ML {model} would yield a desired classification~\cite{Tolomei_17,Wachter_17}. CEs are often advocated as a means to provide algorithmic recourse for data subjects, and as such, modern algorithms have been extended to enforce additional desirable properties such as~\textit{actionability}~\cite{ustun2019actionable}, \textit{plausibility}~\cite{dhurandhar2018explanations} and \textit{diversity}~\cite{mothilal2020explaining}. \JJ{Outside of \AR{the} tabular data domain, they are also widely deployed across image data \cite{DBLP:conf/pkdd/LooverenK21,DBLP:conf/aaai/KennyK21}, textual data \cite{DBLP:conf/acl/WuRHW20,jiang2025interpreting}, etc.}

More central to the MM problem, \citet{pawelczyk2022uai} pointed out that CEs on \AR{the} data manifold are more likely to be robust under MM than minimum-distance CEs. \citet{LeofanteBR23} proposed an algorithm for a given ensemble of feed-forward neural networks to compute robust CEs that are provably valid for all models in the ensemble. Also related to MM is the line of work that focuses on improving the robustness of CEs against model changes, e.g., parameter updates due to model retraining on the same or slightly shifted data distribution.
These studies usually aim to generate CEs that are robust across retrained versions of the same model, e.g., in \JJ{\cite{upadhyay2021towards,pmlr-v162-dutta22a,oursaaai23,hamman2023robust,jiang2023provably,DBLP:journals/ai/JiangLRT24}}, which is different from MM, where several (potentially structurally different) models are targeted together. \JJ{We refer to \cite{DBLP:conf/ijcai/Jiangsurvey24,mishra2021survey,Leofante2025} for recent surveys specifically on the robustness problem of CEs.}


\textbf{Computational Argumentation.}
{This discipline}, 
inspired by the seminal work of 
\cite{Dung_95}, 
{amounts to} a set of formalisms for 
dealing with 
conflicting information
, as demonstrated in numerous application areas, e.g., online debate \cite{Cabrio_13}, scheduling \cite{Cyras_19} and judgmental forecasting \cite{Irwin_22}.
There have also been a broad range of works (see {\cite{Cyras_21,Vassiliades_21,Guo_23}} for overviews) demonstrating 
{its} capability for explaining the outputs of AI models, e.g., neural networks \cite{Potyka_21,Dejl_21}, Bayesian classifiers \cite{Timmer_15} and random forests \cite{Potyka_23}. \JJ{Computational argumentation has also been advocated as a framework to facilitate structured interactions between humans and AI models \cite{DBLP:conf/kr/LeofanteADFGJP024}.} 
To the best of our knowledge, the only work 
{applying computational argumentation} to MM specifically is~\cite{Abchiche-Mimouni_23}, where the introduced method extracts and selects winning rules from an ensemble of classifiers.
This differs from our method as we consider pre-existing CEs computed for single models, rather than 
{rules}, 
{while also accommodating preferences}. 





\section{Preliminaries}
\label{sec:preliminaries}

Given 
a set of classification \emph{labels} $\Classes$, a \emph{model}  is a mapping $M: \mathbb{R}^n \rightarrow \Classes$; we denote that $\model$ classifies an \emph{input} $\inn \in \mathbb{R}^n$ as $\class 
$ iff  $\model(\inn) = \class$. 
{(Note that binary classification amounts to $\Classes = \{ 0, 1 \}$: for simplicity, this will be the focus of all illustrations and of the experiments in Section \ref{sec:evaluation}.)}
Then, a \emph{counterfactual explanation} (CE) for $\inn$, given $\model$, is some $\cfx \in \mathbb{R}^n \setminus \{ \inn \}$ such that $\model(\cfx) \neq \model(\inn)$, which may be optimised by some distance metric between the inputs. 

\JJ{An \emph{abstract argumentation framework} (AAF) \cite{Dung_95} is a tuple $\langle \Args, \Atts\rangle$, where $\Args$ is a set of \emph{arguments}, $\Atts \!\subseteq \!\Args \!\times \!\Args$ is a directed relation of \emph{attack}. Given an AAF $\langle \Args, \Atts\rangle$, for any $\arga \in \Args$, we refer to $\Atts(\arga) = \{\argb | (\argb, \arga) \in \Args\}$ as the \emph{attackers of $\arga$}. \JJ{Let us construct an example AAF for the running loan example in Figure \ref{fig:page1_example}, with each model as an argument. An attack exists between two models if their predictions for the applicant are different. This AAF is visualised in Figure \ref{fig:prelim_example}}. We use the notions of \emph{acceptability} of sets of arguments in AAFs~\cite{Dung_95}\AR{, as follows}.
A set of arguments $X \subseteq \Args$, also called an \emph{extension}, is said to be \emph{conflict-free} iff $\nexists \arga, \argb \in X$ such that $\arga$ attacks $\argb$. A set \AR{$X$} \emph{defends} any $\arga \in \Args$ iff $\forall \argb \in \Args$, if $\argb$ attacks $\arga$ then $\exists \argc \in X$ such that $\argc$ attacks $\argb$. The notion of a set being \emph{admissible} requires that $X$ is conflict-free and defends all its elements. $X$ is said to be \emph{preferred} iff it is the maximal admissible set wrt set-inclusion ($\subseteq$-maximal onwards). We denote the set of all preferred extensions as $P$. \JJ{The admissible sets in our example are $\{M_1\}, \{M_2\}, \{M_3\}$, $\{M_1, M_2\}$, and the preferred extensions are $\{M_1, M_2\}$ and $\{M_3\}$. Note that each preferred extension groups together the models with the same prediction.}}

\begin{figure}[h]
    \centering
    \includegraphics[width=0.3\textwidth]{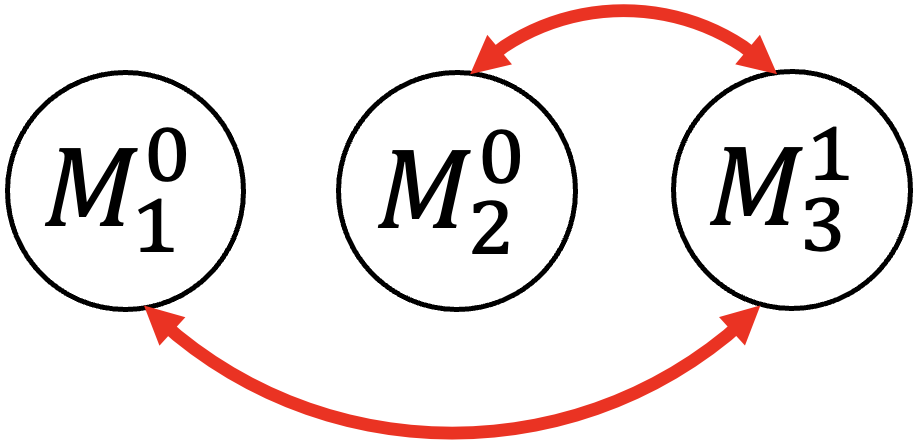}
	\caption{\JJ{AAF for the loan example where model's predictions for the input $\inn$ are given as superscripts, \AR{i.e.}, $M_1(\inn)=0$\AR{, $M_2(\inn)=0$ and} $M_3(\inn)=1$; reciprocal attacks are represented by double-headed red arrows.}
    \label{fig:prelim_example}}
\end{figure}

Similarly, a \emph{bipolar argumentation framework} (BAF) \cite{Cayrol_05} is a tuple $\langle \Args, \Atts, \Supps \rangle$, where 
\AR{$\langle \Args, \Atts \rangle$ is an AAF} and $\Supps\! \subseteq \!\Args \!\times \!\Args$ is a directed relation of \emph{direct support}.
Given a BAF $\langle \Args, \Atts, \Supps \rangle$, for any $\arga \!\in \!\Args$, we refer to 
$\Supps(\arga) \!= \! \{ \argb | (\argb, \arga) \in \Supps \}$ as \emph{$\arga$'s direct supporters}.
Then, an \emph{indirect attack} from $\arg_x$ on $\arg_y$ is a sequence $\arg_1, \rel_1, \ldots, \rel_{n-1}, \arg_n$, where $n \geq 3$, $\arg_1 = \arg_x$, $\arg_n = \arg_y$, $\rel_{1} \in \Atts$ and $\rel_i \in \Supps$ $\forall i \!\in \!\{ 2, \ldots, n - 1 \}$.
Similarly, a \emph{supported attack} from $\arg_x$ on $\arg_y$ is a sequence $\arg_1, \rel_1, \ldots, \rel_{n-1}, \arg_n$, where $n \geq 3$, $\arg_1 = \arg_x$, $\arg_n = \arg_y$, $\rel_{n-1} \in \Atts$ and $\rel_i \in \Supps$ $\forall i \in \{ 1, \ldots, n - 2 \}$.
Straightforwardly, a supported attack on an argument implies 
{a direct attack}. \JJ{Let us construct an example BAF extending the AAF example (Figure \ref{fig:prelim_example}), shown in Figure \ref{fig:prelim_example2}. Assume there are three more arguments 
\AR{$\cfx_1$, $\cfx_2$ and $\cfx_3$}, respectively having bi-directional supports with 
\AR{$M_1$, $M_2$ and $M_3$}. Then, \AR{for example,} there is an indirect attack from $M_3$ to $\cfx_1$, and a supported attack from $\cfx_1$ to $M_3$.}

\begin{figure}[h]
    \centering
    \includegraphics[width=0.3\textwidth]{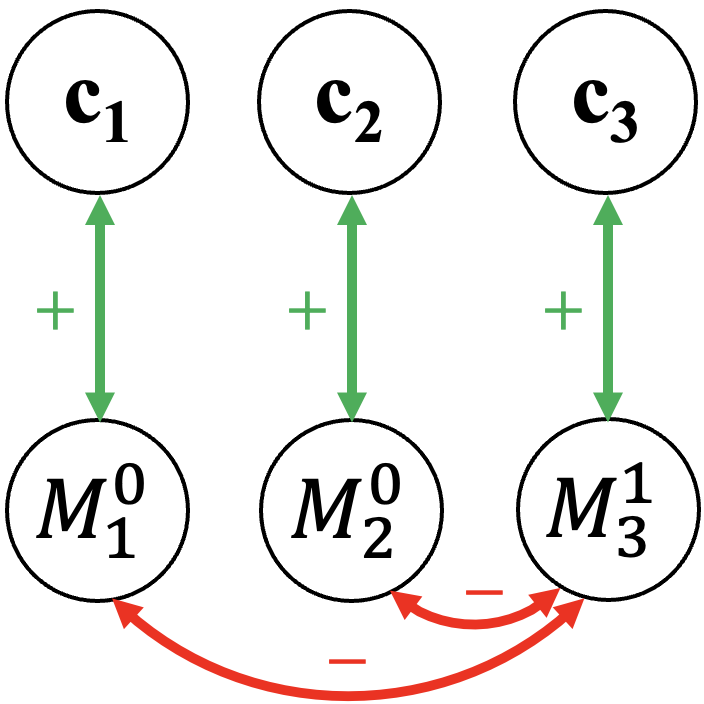}
	\caption{\JJ{An example BAF extending the previous AAF example; reciprocal supports are represented by dual-headed green arrows labelled with $+$ and reciprocal attacks are represented by dual-headed red arrows labelled with $-$.}
    \label{fig:prelim_example2}}
\end{figure}

We will also use notions of acceptability of sets of arguments in BAFs~\cite{Cayrol_05}.
A set of arguments (or extension) $X \subseteq \Args$ is said to \emph{set-attack} any $\arga \in \Args$ iff there exists an attack (whether direct, indirect or supported) from some $\argb \in X$ on $\arga$. 
Meanwhile, $X$ is said to \emph{set-support} any $\arga \in \Args$ iff there exists a direct support from some $\argb \in X$ on $\arga$.\footnote{
In \cite{Cayrol_05}, set-supports are defined 
via 
sequences of supports, 
which we do not use 
here.}
\JJ{For example, in Figure \ref{fig:prelim_example2}, $\{\cfx_1, \cfx_2\}$ set-attacks $M_3$ (because both $\cfx_1$ and $\cfx_2$ supported-attack $M_3$) and set-supports $M_2$.} 
Then, a set $X \subseteq \Args$ \emph{defends} any $\arga \in \Args$ iff $\forall \argb \in \Args$, if $\{ \argb \}$ set-attacks $\arga$ then $\exists \argc \in X$ such that $\{ \argc \}$ set-attacks $\argb$.
Any set $X \subseteq \Args$ is then said to be \emph{conflict-free} iff $\nexists \arga, \argb \in X$ such that $\{ \arga \}$ set-attacks $\argb$, and \emph{safe} iff $\nexists \argc \in \Args$ such that $X$ set-attacks $\argc$ and either: $X$ set-supports $\argc$; or $\argc \in X$. (Note that a safe set is guaranteed to be conflict-free.) \JJ{In the example, the set $\{\cfx_2, \cfx_3\}$ is conflict-free but is not safe, because the two arguments, while not directly attacking each other, simultaneously \AR{supported-}attack and support 
\AR{the same} arguments ($M_2$ and $M_3$)
.} 
\JJ{The notion of a set $X \subseteq \Args$ being \emph{stable} requires that $X$ is conflict-free, and that $X$ set-attacks all arguments which are not in $X$. Similarly, 
the notion of a set $X \subseteq \Args$ being \emph{d-admissible} (based on admissibility in \cite{Dung_95}) 
requires that $X$ is conflict-free and defends all of its elements.}
\JJ{In our running BAF example, the set $\{\cfx_3, M_3\}$ is both stable and d-admissible.}
This notion is extended to account for safe sets: $X$ is said to be \emph{s-admissible} iff it is safe and defends all of its elements. 
({Thus} an s-admissible set is guaranteed to be d-admissible.)
{Further,} $X$ is said to be \emph{c-admissible} iff it is conflict-free, closed for $\Supps$ and defends all of its elements.
Finally, 
$X$ is said to be \emph{d-preferred} (\resp, \emph{s-preferred}, \emph{c-preferred}) iff it is $\subseteq$-maximal d-admissible (\resp, s-admissible, c-\AR{admissible})\JJ{. We denote the set of all \AR{d-preferred (\resp, s-preferred, c-preferred)} extensions as $P^d$ (\resp, $P^s, P^c$). Additionally, we write
the set of all stable extensions as $P^{stable}$. 
Finally, for $\sigma$ amongst any semantics (d-\AR{preferred}, s-\AR{preferred}, c-preferred, and stable), 
we use $P^{\sigma}$ to denote the set of all $\sigma$ extensions.} \JJ{Continuing with the example, the set $\{\cfx_1, \cfx_2, \cfx_3\}$ is stable and d-preferred, but is not s- or c-preferred because it is not safe or closed for support relations. On the contrary, the set $\{\cfx_1, M_1, \cfx_2, M_2\}$ satisfies the requirements for being stable, d-, s-, and c-preferred.}



\section{Recourse under Model Multiplicity}
\label{sec:problem}
 

As mentioned in 
Section \ref{sec:intro}, a common way to deal with MM in practice is to employ ensembling techniques, where the prediction outcomes of several models are aggregated to produce a single outcome. Aggregation can be performed in different ways, as discussed in~\cite{Black_22, black2022selective}. In the following, we formalise a notion of \textit{naive ensembling}, adopted in~\cite{Black_22}, and also known as \textit{majority voting}, which will serve as a baseline for our analysis.\footnote{It should be noted that, in \cite{Black_22}, the case where there is no majority is not discussed.}



\begin{definition}
\label{def:naive_ensembling}
    Given an input $\inn$, a set of models $\Models$ and a set of labels $\Classes$, we define the set of \emph{top labels} $\Classes_{max} \subseteq \Classes$ as:
    
    $$
        \Classes_{max} = argmax_{\class_i \in \Classes} | \{ \model_j \in \Models | \model_j(\inn) = \class_i \} |. \nonumber
   $$
   \\
    Then we then use $\AggModels(\inn) \in \Classes_{max}$ to denote the aggregated classification by \emph{naive ensembling}. 
    In the cases where $| \Classes_{max} | > 1$, we select $\AggModels(\inn)$ from $\Classes_{max}$ randomly
    .
    With an abuse of notation, we also let $\AggModels = \{ \model_j \in \Models | \model_j(\inn) = \AggModels(\inn) \}$ denote the set of models that agree on the aggregated classification.
\end{definition}
%
For illustration, coming back to our loan example where $\model_1$ and $\model_2$ reject the loan ($\model_1(\inn) = \model_2(\inn) = 0$) while $\model_3$ accepts it ($\model_3(\inn) = 1$), we obtain $\AggModels(\inn) = 0$ and $\AggModels = \{ \model_1, \model_2 \}$. 

Naive ensembling is known to be an effective strategy to mediate conflicts between models and is routinely used in practical applications. However, in this paper, we take an additional step and aim to generate CEs providing recourse for a user that has been impacted by $\AggModels(\inn)$. Recent work by~\cite{LeofanteBR23} has shown that standard algorithms designed to generate CEs for single models typically fail to produce recourse recommendations that are robust across $\AggModels$. One natural idea to address this would be to extend naive ensembling to account for CEs. Next, we formalise this idea 
in terms of several properties that 
{we deem} important in this setting. We then 
{analyse two concrete methods extending naive ensembling} in terms of the properties
.

\subsection{Problem Statement and Desirable Properties}
\label{ssec:desirable_properties}

Consider a non-empty set of models $\Models \!\!=\!\!\{\model_1, \ldots, \model_m\!\}$ and, for an input $\inn$, 
assume a set 
$\CFXs(\inn)\!=\!\{\cfx_1, \ldots, \cfx_m\}$ where 
each $\cfx_i \in \CFXs(\inn)$ is 
a CE for $\inn$, given $\model_i$. In the rest of the paper, wherever it is clear that we refer to a given $\inn$, we use $\CFXs$ and omit its dependency on $\inn$ for readability.
Our aim is to solve the problem 
outlined below. 

\myprob{\ProblemFull{} (\Problem{})}{input $\inn$, set  $\Models$ of models, set $\CFXs$ of CEs}{``optimal'' set $S \subseteq \Models \cup \CFXs$ of models and CEs.}

 To characterise 
optimality, we propose a number of desirable properties for the outputs of ensembling methods. We refer to these outputs as \emph{solutions} for \Problem{}. The 
most basic 
requirement 
requires  
that both models and CEs in the output are non-empty. 

\begin{definition}
\label{def:non-emptiness}
    An ensembling method satisfies \emph{non-emptiness} iff for any given input $\inn$, set  $\Models$ of models and set $\CFXs$ of CEs, any 
    solution 
    $S \subseteq \Models \cup \CFXs$ is such that $S \cap \Models \neq \emptyset$ and $S \cap \CFXs \neq \emptyset$. 
\end{definition}

Specifically, 
non-emptiness ensures that 
the \Problem{} 
method returns 
some models and some CEs. 
We then look to ensure that the 
\Problem{} method returns a non-trivial set of models, as formalised 
{next}.

\begin{definition}\label{def:non-triviality}
    An ensembling method satisfies \emph{non-triviality} iff for any given input $\inn$, set  $\Models$ of models and set $\CFXs$ of CEs, any 
    solution $S \subseteq \Models \cup \CFXs$ is such that $| S \cap \Models | > 1$. 
\end{definition}

Clearly, the 
returned models should not disagree amongst themselves on the classification, which leads to our next requirement.



\begin{definition} \label{def:model_agreement}
    An ensembling method satisfies \emph{model agreement} iff for any given input $\inn$, set  $\Models$ of models and set  $\CFXs$ of CEs, any 
    solution $S \subseteq \Models \cup \CFXs$ is such that $\forall \model_i, \model_j \in S \cap \Models$, $\model_i(\inn) = \model_j(\inn)$. 
\end{definition}


The next property, which itself requires model agreement to be satisfied, checks 
whether the 
set of returned models is among the largest of the agreeing sets of models, a motivating property of naive ensembling.

\begin{definition}\label{def:majority_vote}
    An ensembling method satisfies \emph{majority vote} iff it satisfies model agreement and for any given input $\inn$, set  $\Models$ of models, set $\CFXs$ of CEs  and set $\Classes$ of labels, any 
    solution $S \subseteq \Models \cup \CFXs$ is such that, letting $\class_i = \model_j(\inn)$ {for all} $
    \model_j \in S \cap \Models$, $\nexists \class_k \in \Classes \setminus \{ \class_i \}$ such that $| \{ \model_l \in \Models | \model_l(\inn) = \class_k \} | > | \{ \model_
    {l}\in \Models | \model_
    {l}(\inn) = \class_i \} |$.
\end{definition}

Next, we consider the robustness of recourse. Previous work~\cite{LeofanteBR23} considered a very conservative notion of robustness whereby explanations are required to be valid for all models in $\Models$. While this might be desirable in some cases, we highlight that satisfying this property may not always be feasible in practice. We therefore propose a relaxed notion of robustness, which requires that CEs are valid only for the models that support them.

\begin{definition}\label{def:counterfactual_validity}
    An ensembling method satisfies \emph{counterfactual validity} iff for any given input $\inn$, set  $\Models$ of models and set $\CFXs$ of CEs, any 
    solution $S \subseteq \Models \cup \CFXs$ is such that $\forall \model_i \in S \cap \Models$ and $\forall \cfx_j \in S \cap \CFXs$, $\model_i(\cfx_j) \neq \model_i(\inn)$.
\end{definition}

While counterfactual validity is a fundamental requirement for any sound ensembling method, one also needs to ensure that the solutions it generates are coherent, as formalised below. 

\begin{definition}\label{def:explanation-coherence}
    An ensembling method satisfies \emph{counterfactual coherence} iff for any given input $\inn$, set $\Models$ of models and set  $\CFXs$ of CEs, any 
    solution $S \subseteq \Models \cup \CFXs$
    , where $\Models = \{ \model_1, \ldots, \model_m \}$ and $\CFXs = \{ \cfx_1, \ldots, \cfx_m \}$, 
    is such that $\forall i \in \{ 1, \ldots, m \}$, 
    $\model_i \in S$ iff $\cfx_i \in S$.
\end{definition}

Intuitively coherence requires that \textit{(i)} 
a CE is returned only if it is supported by a model and \textit{(ii)} when the 
CE is chosen, then its corresponding model must be part of the support. This ultimately guarantees strong justification as to why a given recourse is suggested since selected models and their reasoning (represented by their CEs) are assessed in tandem. 

The properties defined above may not be all satisfiable at the same time in practice
. Next, we discuss two 
{methods} extending naive ensembling towards solving \Problem{, and explore their satisfaction (or otherwise) of the properties}.

\subsection{
Extending Naive Ensembling for Recourse}
\label{ssec:naive}

We now present two strategies that leverage naive ensembling to solve \Problem{}. In particular, we use the relationship between models in the ensemble $\AggModels$ and their corresponding CEs as follows.

 \begin{definition}
\label{def:baselinecfxs}
\label{def:cfx_selection_methods}
    Consider an input $\inn$, a set $\Models$ of models and a set  $\CFXs$ of CEs. Let $\AggModels \subseteq \Models$ be the set of models obtained by naive ensembling. We define 
        the set of \emph{naive 
        CEs} 
        as:
        
       $$
                    \CFXs^n = \{ \cfx_i \in  \CFXs \mid \model_i \in \AggModels \};
       $$
       \\
        and the set of \emph{valid 
        CEs} as:
        
       $$
            \CFXs^{v} = \{ \cfx_i \in  \CFXs \mid \model_i \in \AggModels 
            \wedge \forall \model_j \in \AggModels, \model_j(\cfx_i) \neq \model_j(\inn) \}.
        $$
        \\
    Then, two possible solutions to \Problem{} are $S^n = \AggModels \cup \CFXs^n$, named \emph{augmented ensembling}, or $S^v = \AggModels \cup \CFXs^{v}$, named \emph{robust ensembling}.
\end{definition}

Intuitively, 
augmented ensembling suggests 
taking all the CEs in $\CFXs$ that correspond to the models 
in $\AggModels$. 
Meanwhile, robust ensembling extends augmented ensembling by enforcing the additional constraint that CEs are selected only if they are valid for all models in $\AggModels$. We now provide an illustrative example to clarify the results produced by the two strategies.

\begin{example}
\label{ex:baselines_example}
    Consider 
    $\Models = \{ \model_1, \model_2, \model_3, \model_4, \model_5 \}$ and an input $\inn$ such that $\model_1(\inn) = \model_2(\inn) = \model_3(\inn) = 0$ and $\model_4(\inn) = \model_5(\inn) = 1$. Let $\CFXs = \{\cfx_1, \cfx_2, \cfx_3, \cfx_4, \cfx_5\}$ be the set of CEs generated for $\inn$, i.e., $\model_1(\cfx_1) = \model_2(\cfx_2) = \model_3(\cfx_3) = 1$, while $\model_4(\cfx_4) = \model_5(\cfx_5) = 0$. Applying naive ensembling to $\Models$ yields $\AggModels = \{\model_1, \model_2, \model_3\}$ and $\AggModels(\inn)=0$. Then, the set of naive CEs is $\CFXs^n = \{\cfx_1, \cfx_2, \cfx_3\}$, and thus augmented ensembling gives $S^n = \{\model_1, \model_2, \model_3, \cfx_1, \cfx_2, \cfx_3\}$. 
    Now, assume that $\cfx_1$ is invalid for $\model_2$ (i.e., $\model_2(\cfx_1) = 0$), $\cfx_2$ is invalid for $\model_1$, $\cfx_3$ is invalid for $\model_2$, and all three CEs are otherwise valid for all models in $\AggModels$. Then, then the set of valid CEs 
    is empty, i.e., $\CFXs^{v} = \emptyset$, and thus robust ensembling gives $S^v = \{\model_1, \model_2, \model_3 \}$.
\end{example}

This example shows that both methods host major drawbacks: augmented ensembling may produce CEs which are invalid and thus it is not robust to  MM, while robust ensembling is prone to returning 
no CEs. 
We now present a theoretical analysis to assess the extent to which augmented and robust ensembling are able to satisfy the properties given in Definitions~\ref{def:non-emptiness} to \ref{def:explanation-coherence}.

\begin{theorem}
\label{thm:s1}
    Augmented ensembling satisfies non-emptiness, model agreement, majority vote and counterfactual coherence. It satisfies non-triviality if $|\Models| \!>\! 2$. It does not satisfy counterfactual validity.
\end{theorem}

\begin{proof}
    By Definition~\ref{def:naive_ensembling}, it can be seen by inspection that $|\AggModels|\!\!>\!\! 0$. Thus, non-emptiness is satisfied.
    Again 
    by inspection of the same definition, 
    $\forall \model_i, \model_j \in \AggModels$, $\model_i(\inn) = \model_j(\inn)$. Thus, model agreement is satisfied. 
    We can also see that $\nexists \class_i \in \Classes \setminus \{ \AggModels(\inn) \}$ such that $| \{ \model_j \in \Models | \model_j(\inn) = \class_i \} | > | \AggModels |$. Thus, majority vote is satisfied. 
    By Definitions~\ref{def:naive_ensembling} and \ref{def:baselinecfxs}, it can be seen that $\forall \model_i \in \Models$ and $\forall \cfx_i \in \CFXs$, $\model_i \in S$ iff $\cfx_i \in S$. Thus, counterfactual coherence is satisfied.

    Example \ref{ex:baselines_example}
     shows that counterfactual validity is not satisfied by providing a counterexample.
    
    Finally, the partial satisfaction of non-triviality can be proven by contradiction: assume $|\Models| \!=\! n$,  $n \!>\! 2$ but $|\AggModels| \!=\! 1$. By Definition~\ref{def:naive_ensembling}, $\AggModels$ is the largest subset of $\Models$ containing models with the same classification outcome. However, for binary classification, this implies that the remaining $n-1$ all agree on the opposite classification, i.e., $|\Models \!\setminus\! \AggModels| \!>\! |\AggModels| $, which leads to a contradiction.
\end{proof}

\begin{theorem}
\label{thm:s2}
    Robust ensembling satisfies model agreement, majority vote and counterfactual validity. It satisfies non-triviality if $|\Models| > 2$. It does not satisfy non-emptiness or counterfactual coherence.
\end{theorem}

\begin{proof}
    The proofs for model agreement, majority vote and non-triviality are analogous to those in Theorem \ref{thm:s1} and so are omitted.

    It can be seen by inspection of Definition \ref{def:baselinecfxs} that $\forall \!\model_i \!\in\! \AggModels$ and $\forall \cfx_j \!\in \!\CFXs^v$, $\model_i(\cfx_j) \!\neq \!\model_i(\inn)$. Thus, counterfactual validity is satisfied.

    Example \ref{ex:baselines_example} shows that non-emptiness and counterfactual coherence are not satisfied by providing a counterexample.
\end{proof}

These results, summarised in Table \ref{table:props},
demonstrate that there may exist cases in which both augmented and robust ensembling fail to solve \Problem{} satisfactorily. This has strong implications on the quality of the results obtained in practice, as we will show experimentally in Section \ref{sec:evaluation}. Further, these methods provide no way to take into account users' preferences over the models. 
As previously mentioned (see Section \ref{sec:intro} and Section \ref{sec:related}), there could be different characteristics among the models in $\Models$ in terms of meta-evaluation aspects, e.g., a model's fairness, robustness, and simplicity. Depending on the task, a model's fairness might be specified as being more important than its robustness or simplicity. In such cases, it would be desirable to have a principled way to rank models according to the preference specification, as promoted in \cite{Black_22}. 
The combination of these deficiencies motivates the need for a richer ensembling framework to solve \Problem{} while incorporating user preferences, 
given next.

\begin{table}[t]
    \begin{center}
    \begin{tabular}{cccc}
    \cline{2-4}
     &
    $S^{n}$ &
    $S^{v}$ &
    $S^{a,s}$ \\
    \hline
    non-emptiness &
    $\checkmark$ &
     &
    $\checkmark$ \\
    non-triviality &
    $\checkmark^*$ &
    $\checkmark^*$ &
    $\checkmark^*$ \\
    model agreement &
    $\checkmark$ &
    $\checkmark$ &
    $\checkmark$ \\
    majority vote &
    $\checkmark$ &
    $\checkmark$ &
     \\
    counterfactual validity &
     &
    $\checkmark$ &
    $\checkmark$ \\
    counterfactual coherence &
    $\checkmark$ &
     &
    $\checkmark$ \\
    \hline
    \end{tabular}
    \end{center}
    \protect\caption{Augmented ($S^n=\AggModels \cup \CFXs^{n}$) and robust ($S^v=\AggModels \cup \CFXs^{v}$) ensembling, as well as our argumentative approach \JJ{with \AR{the} s-preferred extension} (defined later in Section \ref{sec:main}), assessed against the desirable properties defined in Section \ref{ssec:desirable_properties}. Satisfaction of a property is 
    shown
    by $\checkmark$, while partial satisfaction under given conditions is 
    shown
    by $\checkmark^*$.} 
    \label{table:props}
    \end{table}


\section{Argumentative Ensembling}
\label{sec:main}

\JJ{In this section, we introduce our method for ensembling models and CEs which inherently supports specifying preferences over models. We first present general definitions in Section \ref{ssec:general_definition} with illustrative examples, then we instantiate argumentative ensembling with four different argumentation semantics and undertake theoretical analyses of their properties (Section \ref{ssec:d_extension_theoretical} to Section \ref{ssec:ensembling_behaviours}). We finally discuss alternative definitions of argumentative ensembling in Section \ref{ssec:aaf}.} 


\subsection{Definition}
\label{ssec:general_definition}

We start by formalising 
ways to incorporate the 
aforementioned preferences over models.
These preferences could be obtained by any information, e.g., meta-rules over 
models as suggested in \cite{Black_22}, but we will assume that they are extracted wrt 
properties of the models, e.g., their accuracy or (a metric representing) their simplicity.

\begin{definition}
\label{def:prop}
     Given a set  $\Models$ of models, a set $\Props$ of \emph{properties} is such that $\forall \prop \in \Props$, $\prop: \Models \rightarrow \mathbb{R}$ is a total function.
\end{definition} 

We then define a preference over the properties such that users can impose a ranking of priority over them. Here and onwards, for simplicity we use total orders, denoted $\preceq$, over any set $S$ such that, as usual, for any $s_i, s_j \in S$, $s_i \prec s_j$ iff $s_i \preceq s_j$ and $s_i \not{\succeq} s_j$. Also as usual, we say that $s_i \simeq s_j$ iff $s_i \preceq s_j$ and $s_i \succeq s_j$.

\begin{definition}
\label{def:proppref}
    Given a set  $\Props$ of properties, a \emph{property preference} $\ppreceq$ is a total order over $\Props$.
\end{definition}

Model preferences 
can be defined using property preferences.
In the following example, we define one way for 
doing so.

\begin{example}
\label{ex:CEs}
    Consider the same 
    models as in Example \ref{ex:baselines_example} and a set of properties $\Props \!\!=\!\! \{ \prop_1, \prop_2 \}$ 
    where $\prop_1, \prop_2$ represent 
    model accuracy and simplicity, \resp, with a property preference $\ppreceq$ such that $\prop_1 \!\!\psucc\!\! \prop_2$ and values for the {satisfaction of} properties as follows. 

    \begin{table}[h]
    \begin{center}
    \begin{tabular}{cccccc}
    \cline{2-6}
    &
    $\model_1$ &
    $\model_2$ &
    $\model_3$ &
    $\model_4$ &
    $\model_5$ \\
    \hline
    $\prop_1$ (accuracy) &
    0.85 &
    0.87 &
    0.86 &
    0.86 &
    0.87 \\
    $\prop_2$ (simplicity) &
    0 &
    0.75 &
    1 &
    0.5 &
    0.75 \\
    \hline
    \end{tabular}
    \end{center}
    \label{table:example}
    \end{table}
    
    A simple model preference $\mpreceq$ over $\Models$ may be such that, for any $\model_i, \model_j \in \Models$, $\model_i \msucc \model_j$ iff:
        (i) $\prop_1(\model_i) > \prop_1(\model_j)$; or (ii) $\prop_1(\model_i) = \prop_1(\model_j)$ and $\prop_2(\model_i) > \prop_2(\model_j)$.
    This $\mpreceq$ is a total order over $\Models$ and results in 
    $\model_2 \msimeq \model_5 \msucc \model_3 \msucc \model_4 \msucc \model_1$.
\end{example}

Other ways to define preferences over models from preferences over properties of models can be defined, e.g., based on more sophisticated notions of dominance. We will assume some given notion of total preference over models, as follows, ignoring how it is obtained. 

\begin{definition}
\label{def:modelpref}
    Given a set  $\Models$ of models, 
    a \emph{model preference} $\mpreceq$ over $\Models$ is a total order over $\Models$.
\end{definition}

How can we incorporate these model preferences into the ensembling, while still satisfying the properties defined in Section \ref{ssec:desirable_properties}? To tackle this problem, we 
use bipolar argumentation\footnote{\JJ{We discuss the similarities and differences to abstract argumentation in Section \ref{ssec:aaf}.}} as follows.

\begin{definition}
\label{def:BAF}
    The \emph{BAF corresponding to} 
    input $\inn$, 
    set  $\Models 
    $ of models, 
    set  $\CFXs 
    $ of CEs and 
    model preference $\mpreceq\!$ 
     is $\langle \Args, \Atts, \Supps \rangle$ 
    with:
    \begin{itemize}
        \item $\Args = \Models \cup \CFXs$;
        \item $\Atts \subseteq ( \Models \times \Models) \cup ( \Models \times \CFXs) \cup ( \CFXs \times \Models)$ where:
        \begin{itemize}
            %
            %
            \item 
            $\forall \model_i, \model_j \!\in \!\Models$, $(\model_i,\model_j) \!\in\! \Atts$ iff $\model_i(\inn) \!\!\neq \!\!\model_j(\inn)$
            , $\model_i \!\msucceq \!\model_j$;
            %
            %
            \item 
            $\forall \model_i \!\in\! \Models$ and 
            $\cfx_j \!\in\! \CFXs$ 
            where $\model_i(\cfx_j) \!= \!\model_i(\inn)$, $(\model_i, \cfx_j) \in \Atts$ iff $\model_i \msucceq \model_j$ and $(\cfx_j, \model_i) \in \Atts$ iff $\model_j \msucceq \model_i$;
            
            %
            %
        \end{itemize}
        \item $\Supps \subseteq (\Models \times \CFXs) \cup (\CFXs \times \Models)$ where for any $\model_i \in \Models$ and 
        $\cfx_j \in \CFXs$, $(\model_i,\cfx_j), (\cfx_j,\model_i) \in \Supps$ iff $i=j$.
    \end{itemize}
\end{definition}

Here, a model attacks another model if they disagree on the prediction and the latter is not strictly preferred to the former. This means that models which are outperformed with regards to preferences must be defended by more-preferred, agreeing models in order to be considered acceptable. Models and CEs are treated similarly, with the CEs inheriting the preferences from the models by which they were generated, and attacks being present between them when the model considers the CE invalid. This, along with the fact that models and their CEs support one another, ensures that the models are inherently linked to their reasoning, in the form of their CEs, and conflicts are drawn not only when two models' predictions differ, but also when their reasoning differs.

\JJ{Next, we define argumentative ensembling for solving the MM problem described in the corresponding BAF $\langle \Args, \Atts, \Supps \rangle$.}

    

\begin{definition}
\label{def:argensembling}
\JJ{
    Consider an input $\inn$, a set $\Models$ of models, a set  $\CFXs$ of CEs, a set $\Classes$ of labels and a model preference $\mpreceq$.
    Let $S^{a,\sigma}$ be 
    a 
    cardinality-maximal 
    element of $P^{\sigma}$ 
    for the corresponding BAF $\langle \Args, \Atts, \Supps \rangle$, according to some semantics $\sigma$, i.e.,
    \[
        S^{a,\sigma} 
        \in argmax_{X \in P^{\sigma}} |X|. 
    \]
    Then, a solution to \Problem{} by \emph{$\sigma$ argumentative ensembling} is 
    $S^{a,\sigma}$
    . 
    }
\end{definition}

\JJ{The intuition is that, after modelling the conflicts under MM with a BAF, acceptability-based argumentation semantics can be applied to the BAF to obtain subsets of models and CEs which satisfy some argumentative properties, such as conflict-freeness. While these are satisfied by all extensions in $P^\sigma$ according to the semantics $\sigma$, in Definition \ref{def:argensembling} we only return the cardinality-maximal extension in the spirit of better satisfaction of majority vote. Note that there may be multiple extensions which are all cardinality-maximal ($|P^\sigma| > 1$, e.g., when one extension is $\{\model_1, \model_2, \cfx_1, \cfx_2\}$ and the other is $\{\model_3, \model_4, \cfx_3, \cfx_4\}$ under a four-model scenario). In this case, we randomly choose one as the solution to RAE. 
Alternatively, we could choose to report all viable resulting ensembles in $P^\sigma$, \AR{potentially allowing for} more informed decisions \AR{to} be made by relevant stakeholders. We leave this to future work.}

The following example demonstrates how quickly the problem, when preferences are included, can become complex. This is the case even with only five models, far fewer than usual in MM. \JJ{For illustrative purposes, we instantiate the argumentative ensembling solutions with \AR{the} s-preferred semantics.}

\begin{example}
\label{ex:arg}
        The  BAF corresponding to input, models and CEs as in Example \ref{
        ex:CEs} is $\langle \Args, \Atts, \Supps \rangle$ with (see  Fig.~\ref{fig:example_BAF}):
    $\Args = $
    $\{ \model_1,$ 
    $\model_2,$ 
    $\model_3,$ 
    $\model_4,$ 
    $\model_5,$ 
    $\cfx_1,$ 
    $\cfx_2,$ 
    $\cfx_3,$ 
    $\cfx_4,$ 
    $\cfx_5 \}$; 
    $\Atts = $ 
    $\{(\model_2,\model_4), $
    $(\model_2,\model_5),$
    $(\model_2,\cfx_1),$ 
    $(\model_2,\cfx_3),$ 
    $(\model_3,\model_4\!),$
    $(\model_3,\cfx_4\!),$ 
    $(\model_4,\model_1\!),$
    $(\model_5,\model_1\!),$
    $(\model_5,\model_2\!),$
    $(\model_5,\model_3\!),$
    $(\cfx_2,\model_1\!)\!\}$; and
    $\Supps = $ 
    $\{(\model_1,\cfx_1),$ 
    $(\model_2,\cfx_2),$ 
    $(\model_3,\cfx_3),$ 
    $(\model_4,\cfx_4),$ 
    $(\model_5,\cfx_5),$  
    $(\cfx_1, \model_1),$ 
    $(\cfx_2, \model_2),$  
    $(\cfx_3, \model_3),$ 
    $(\cfx_4, \model_4),$ 
    $(\cfx_5, \model_5)\}$.
    \JJ{With \AR{the} s-preferred semantics, t}his leads to 
    $P^s = $
    $\{ \{ \model_2,$ 
    $\cfx_2 \},$ 
    $\{ \model_4,$ 
    $\model_5,$ 
    $\cfx_4,$ 
    $\cfx_5 \} \} $, and thus 
    $S^{a,s} = $
    $\{ \model_4,$ 
    $\model_5,$ 
    $\cfx_4,$ 
    $\cfx_5 \}$, \JJ{the aggregated prediction} $\model_4(\inn) = \model_5(\inn) = 1$.
\end{example}

\begin{figure}[h]
    \centering
    \includegraphics[width=0.55\textwidth,trim={0.7cm 0.7cm 0.7cm 0.7cm},clip]{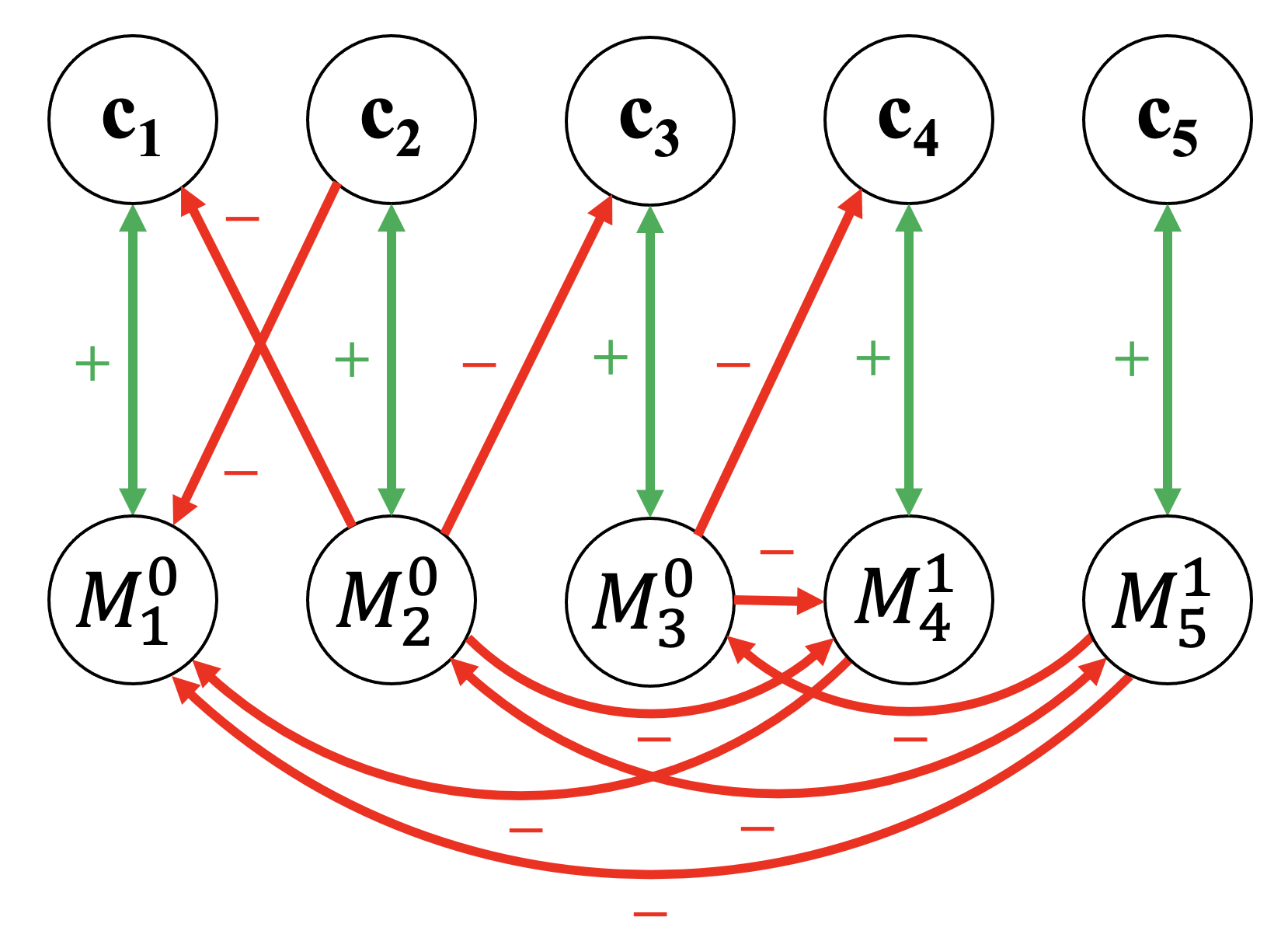}
	\caption{BAF for Example \ref{ex:arg} where: models' predictions for the input $\inn$ are given as superscripts, e.g., $\model_1(\inn) = 0$ but $\model_4(\inn) = 1$; reciprocal supports are represented by dual-headed green arrows labelled with $+$ and standard (reciprocal) attacks are represented by single-headed (dual-headed, \resp) red arrows labelled with $-$. 
    \label{fig:example_BAF}}
\end{figure}

This example shows how the use of  CEs in the ensembling directly results in the prediction being reversed, relative to the other ensembling methods, violating majority vote. This is due to the fact that, while the preferences over the two sets of models are roughly similar, the validity of the CEs for the models selected by the augmented and robust ensemblings is very poor. This means that when a CE 
{that} is valid for all models is required, some compromise must be made on the model selection, as we demonstrated.

\JJ{Definition \ref{def:argensembling} remains agnostic to the argumentative semantics applied to the BAF. In this work, we consider the acceptability-based semantics in \cite{Cayrol_05}, and illustrate the similarities and differences of \emph{d-preferred}, \emph{stable}, \emph{s-preferred}, and \emph{c-preferred} semantics with Examples \ref{ex:representative_example_1} and \ref{ex:representative_example_2}, two simple yet representative scenarios where three models agree on the prediction result but there are CE validity-related conflicts between them.}

\begin{figure}
     \centering
     \begin{subfigure}{0.35\textwidth}
         \centering
         \includegraphics[width=\textwidth]{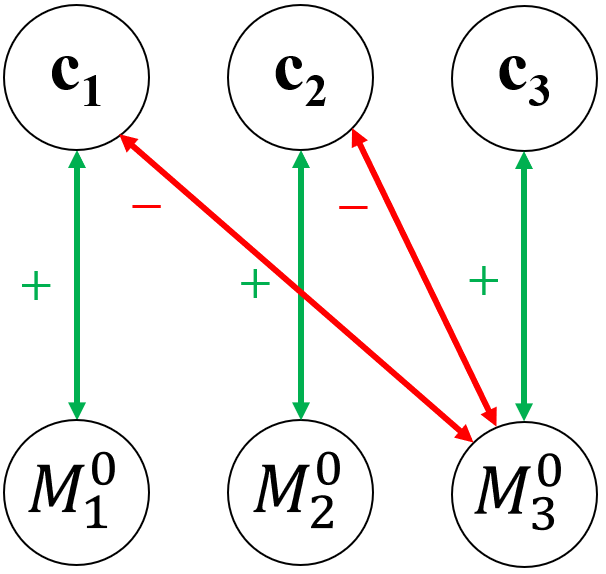}
         \caption{Graphical illustration for Example \ref{ex:representative_example_1}}
         \label{fig:rep1}
     \end{subfigure}
     \hfill
     \begin{subfigure}{0.35\textwidth}
         \centering
         \includegraphics[width=\textwidth]{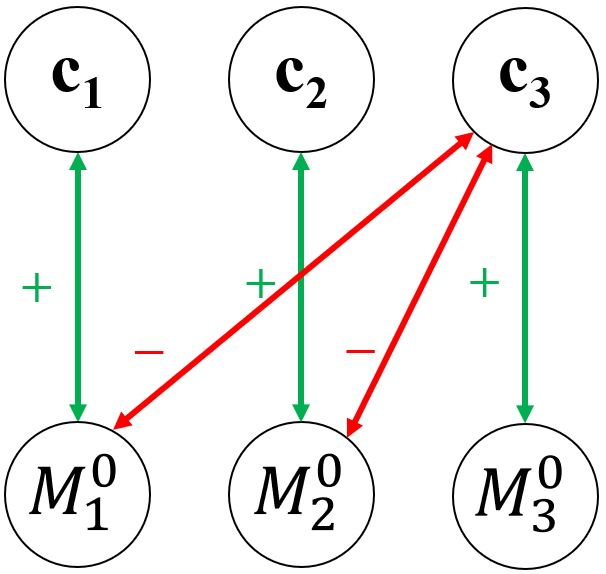}
         \caption{Graphical illustration for Example \ref{ex:representative_example_2}}
         \label{fig:rep2}
     \end{subfigure}
        \caption{\JJ{Two illustrative examples.}}
        \label{fig:rep_example}
\end{figure}

\begin{example}
\label{ex:representative_example_1}
    \JJ{Consider $\Models = \{ \model_1, \model_2, \model_3\}$ and an input $\inn$ such that $\model_1(\inn) = \model_2(\inn) = \model_3(\inn) = 0$. Let $\CFXs = \{\cfx_1, \cfx_2, \cfx_3\}$ be the set of CEs generated for $\inn$, i.e., $\model_1(\cfx_1) = \model_2(\cfx_2) = \model_3(\cfx_3) = 1$. The three models are equally preferred. Assume the CEs $\cfx_1, \cfx_2$ are not valid on $\model_3$, the BAF corresponding to this scenario is: $\langle \Args, \Atts, \Supps \rangle$ with: $\Args = $
    $\{ \model_1,$ 
    $\model_2,$ 
    $\model_3,$ 
    $\cfx_1,$ 
    $\cfx_2,$ 
    $\cfx_3\}$; 
    $\Atts = $ 
    $\{(\cfx_1,\model_3), $
    $(\model_3,\cfx_1), $
    $(\cfx_2,\model_3), $
    $(\model_3,\cfx_2)\}$; and
    $\Supps = $ 
    $\{(\model_1,\cfx_1),$ 
    $(\model_2,\cfx_2),$ 
    $(\model_3,\cfx_3),$ 
    $(\cfx_1, \model_1),$ 
    $(\cfx_2, \model_2),$  
    $(\cfx_3, \model_3)\}$. See Figure \ref{fig:rep1} for a graphical illustration.}
\end{example}

\begin{example}
\label{ex:representative_example_2}
    \JJ{Consider $\Models = \{ \model_1, \model_2, \model_3\}$ and an input $\inn$ such that $\model_1(\inn) = \model_2(\inn) = \model_3(\inn) = 0$. Let $\CFXs = \{\cfx_1, \cfx_2, \cfx_3\}$ be the set of CEs generated for $\inn$, i.e., $\model_1(\cfx_1) = \model_2(\cfx_2) = \model_3(\cfx_3) = 1$. The three models are equally preferred. Assume the CE $\cfx_3$ is not valid on $\model_1$ and $\model_2$, the BAF corresponding to this scenario is: $\langle \Args, \Atts, \Supps \rangle$ with: $\Args = $
    $\{ \model_1,$ 
    $\model_2,$ 
    $\model_3,$ 
    $\cfx_1,$ 
    $\cfx_2,$ 
    $\cfx_3\}$; 
    $\Atts = $ 
    $\{(\cfx_3,\model_1), $
    $(\model_1,\cfx_3), $
    $(\cfx_3,\model_2), $
    $(\model_2,\cfx_3)\}$; and
    $\Supps = $ 
    $\{(\model_1,\cfx_1),$ 
    $(\model_2,\cfx_2),$ 
    $(\model_3,\cfx_3),$ 
    $(\cfx_1, \model_1),$ 
    $(\cfx_2, \model_2),$  
    $(\cfx_3, \model_3)\}$. See Figure \ref{fig:rep2} for a graphical illustration.}
\end{example}

\begin{table}[h]
    \centering
    \begin{tabular}{c|c|c}

        & Example \ref{ex:representative_example_1} & Example \ref{ex:representative_example_2} \\
        \hline
        \multirow{2}{*}{stable} & $S^{a,stable}=\{\model_1, \model_2, \cfx_1, \cfx_2\}$, & $S^{a,stable}=\{\model_1, \model_2, \cfx_1, \cfx_2\}$, \\
        
         & $\{\model_1, \model_2, \cfx_3\}$, $\{\model_3, \cfx_3\}$ & $\{\model_3, \cfx_1, \cfx_2\}$, $\{\model_3, \cfx_3\}$ \\

        \hline
        
        \multirow{2}{*}{d-preferred} & $S^{a,d}=\{\model_1, \model_2, \cfx_1, \cfx_2\}$, & $S^{a,d}=\{\model_1, \model_2, \cfx_1, \cfx_2\}$, \\
        
         & $\{\model_1, \model_2, \cfx_3\}$, $\{\model_3, \cfx_3\}$ & $\{\model_3, \cfx_1, \cfx_2\}$, $\{\model_3, \cfx_3\}$ \\

         \hline
         
        \multirow{2}{*}{s-preferred} & $S^{a,s}=\{\model_1, \model_2, \cfx_1, \cfx_2\}$, & $S^{a,s}=\{\model_1, \model_2, \cfx_1, \cfx_2\}$, \\
        
         & $\{\model_3, \cfx_3\}$ & $\{\model_3, \cfx_3\}$ \\

         \hline
         
        \multirow{2}{*}{c-preferred} & $S^{a,c}=\{\model_1, \model_2, \cfx_1, \cfx_2\}$, & $S^{a,c}=\{\model_1, \model_2, \cfx_1, \cfx_2\}$, \\
        
         & $\{\model_3, \cfx_3\}$ & $\{\model_3, \cfx_3\}$ \\

        \hline
    \end{tabular}
    \caption{\JJ{Different semantics applied to Examples \ref{ex:representative_example_1} and \ref{ex:representative_example_2}}. $P^{stable}, P^d, P^s, P^c$ are in each table entry, with argumentative ensembling \AR{solutions} highlighted within. }
    \label{tab:semantics_representative_example}
\end{table}

\JJ{We list $P^{stable}$, $P^d$, $P^s$ \AR{and} $P^c$ according to each semantics, as well as the corresponding argumentative ensembling results in Table \ref{tab:semantics_representative_example}. We observe that all semantics return the same set $\{\model_1, \model_2, \cfx_1, \cfx_2\}$ as the ensembling solution. Thus, argumentative ensembling methods based on these semantics would potentially satisfy all the requirements described in the desirable properties (Section \ref{ssec:desirable_properties}). Therefore, intuitively, all semantics can lead to meaningful solutions. However, the accepted sets, which are the candidates for the final ensembling solution, are different. It should be noted that, since the stable and d-preferred extensions accept $\{\model_1, \model_2, \cfx_3\}$, argumentative ensembling methods computed using these extensions would satisfy all (requirements of the) properties except counterfactual coherence.}


\JJ{In the following \AR{sub}sections, we perform \AR{a} theoretical analysis \AR{examining} 
four argumentative semantics applied to Definition \ref{def:argensembling} \AR{against the desirable properties from Section \ref{ssec:desirable_properties}}.}

\subsection{\JJ{D-preferred and Stable Argumentative Ensembling}}
\label{ssec:d_extension_theoretical}

\JJ{We start with the least strict d-preferred argumentative ensembling. We show that although it satisfies important properties of model agreement and counterfactual validity, \AR{but that} it does not satisfy the rest.}

\begin{theorem}
\label{thm:d_preferred_results}
    \JJ{D-preferred argumentative ensembling satisfies model agreement \AR{and} counterfactual validity. 
    It does not satisfy non-emptiness, non-triviality, majority vote \AR{or} counterfactual coherence.}
\end{theorem}

\begin{proof}


    \JJ{Let us 
    prove model agreement by contradiction. Assume that $\exists \model_i, \model_j \!\!\in\! \Models$ such that $\model_i(\inn) \!\neq\! \model_j(\inn)$ and $\exists P^d_{i} \!\!\in\! P^d$ such that $\model_i, \model_j \!\in\! P^d_{i}$. By Definition \ref{def:BAF}, it follows that $\exists (\model_i, \model_j) \!\in\! \Atts$ or $\exists (\model_j, \model_i) \in \Atts$, which cannot be the case in a d-preferred set, which must be conflict-free (see Section \ref{sec:preliminaries}), and so we have the contradiction. }

    \JJ{Let us prove counterfactual validity by contradiction. Assume that $\exists \model_i \!\in\! \Models$ and $\exists \cfx_j \!\in\! \CFXs$ such that $\model_i(\cfx_j) \!=\! \model_i(\inn)$ and $\exists P^d_{i} \!\in\! P^d$ such that $\model_i, \cfx_j \!\in\! P^d_{i}$. By Definition \ref{def:BAF}, it can be seen that $\exists (\model_i, \cfx_j) \!\in\! \Atts$ or $\exists (\cfx_j, \model_i) \!\in\! \Atts$, which cannot be the case in a d-preferred set, which, again, must be conflict-free, and so we have the contradiction.}

    \JJ{We show that non-emptiness is not satisfied by providing the following counterexample. Consider $\Models = \{ \model_1, \model_2\}$ and an input $\inn$ such that $\model_1(\inn) = 0, \model_2(\inn) = 1$. Let $\CFXs = \{\cfx_1, \cfx_2\}$ be the set of CEs generated for $\inn$, with $\model_1(\cfx_1) = 1, \model_2(\cfx_2) = 0$, but $\model_1(\cfx_2) = 1, \model_2(\cfx_1) = 0$. The two models are equally preferred. The BAF corresponding to this scenario is: $\langle \Args, \Atts, \Supps \rangle$ with: $\Args = $
    $\{ \model_1,$ 
    $\model_2,$ 
    $\cfx_1,$ 
    $\cfx_2\}$; 
    $\Atts = $ 
    $\{(\model_1,\model_2), (\model_2,\model_1)\}$; and
    $\Supps = $ 
    $\{(\model_1,\cfx_1),$ 
    $(\model_2,\cfx_2),$ 
    $(\cfx_1, \model_1),$ 
    $(\cfx_2, \model_2)\}$. Then, $P^d=\{\{\cfx_1,\cfx_2\}, \{\model_1,\cfx_1\}, \{\model_2,\cfx_2\}\}$. \AR{The} d-preferred argumentative ensembling solution $S^{a,d}$ can therefore be $\{\cfx_1,\cfx_2\}$, which, by Definition \ref{def:non-emptiness}, violates non-emptiness as no model is included.}

    \AR{Non-triviality is violated with the same counterexample as counterfactual coherence.}

    \JJ{We show that majority vote is not satisfied by providing the following counterexample. Consider the BAF in Example \ref{ex:arg}. With d-preferred extension, $S^{d}=\{\model_4, \cfx_4, \model_5, \cfx_5\}$, which is also the d-preferred argumentative ensembling solution. In this case, $\forall \model_i \in S^d \cap \Models$, $\model_i(x) = 1$. However, $| \{ \model_j \in \Models | \model_j(\inn) = 0 \} | > | \{ \model_i\in \Models | \model_
    {i}(\inn) = 1 \} |$. Therefore, by Definition \ref{def:majority_vote}, majority vote is violated.
    }

    \JJ{We show that counterfactual coherence is not satisfied by providing the following counterexample. Consider Example \ref{ex:representative_example_1} but additionally $\cfx_2$ is not valid on $\model_1$. The BAF corresponding to this scenario is: $\langle \Args, \Atts, \Supps \rangle$ with: $\Args = $
    $\{ \model_1,$ 
    $\model_2,$ 
    $\model_3,$ 
    $\cfx_1,$ 
    $\cfx_2,$ 
    $\cfx_3\}$; 
    $\Atts = $ 
    $\{(\cfx_2,\model_1), $
    $(\model_1,\cfx_2), $
    $(\cfx_1,\model_3), $
    $(\model_3,\cfx_1), $
    $(\cfx_2,\model_3), $
    $(\model_3,\cfx_2)\}$; and
    $\Supps = $ 
    $\{(\model_1,\cfx_1),$ 
    $(\model_2,\cfx_2),$ 
    $(\model_3,\cfx_3),$ 
    $(\cfx_1, \model_1),$ 
    $(\cfx_2, \model_2),$  
    $(\cfx_3, \model_3)\}$.
    Then, $P^d=\{\{\model_1,\cfx_1\}, \{\model_2,\cfx_2\}, \{\model_3,\cfx_3\}, \{\model_1,\cfx_3\}, \{\model_2,\cfx_1\}, \\ \{\model_2,\cfx_3\} \}$. \AR{The} d-preferred argumentative ensembling solution $S^{a,d}$ can therefore be \AR{any of} $\{\model_1,\cfx_3\}, \{\model_2,\cfx_1\}, \{\model_2,\cfx_3\}$, \AR{all of} which 
    violate counterfactual coherence (Definition \ref{def:explanation-coherence}).}

\end{proof}

\JJ{The \AR{violation} of counterfactual coherence and non-emptiness hinders the reliability of d-preferred argumentative ensembling. Although practically, any CE and model can be included in the solution set as long as the CE is valid on that model, this might create imbalances in the number of models and CEs, causing \AR{the} exclusion of any model (\AR{e.g., as in }the counterexample for non-emptiness \AR{in the} proof \AR{of Theorem \ref{thm:d_preferred_results}}), or fewer CEs, causing the set of explanations to be less diverse.
However, our empirical evaluations in Section \ref{sec:evaluation} find that these properties are almost always satisfied in practice.}


\JJ{For an extension to be stable, in addition to 
\AR{being conflict-free}, it also needs to set-attack all arguments that \AR{it does not contain}. Next, we show that under our definition of BAF, this stricter requirement is automatically satisfied by \AR{the} d-preferred semantics. Therefore, the results in Theorem \ref{thm:d_preferred_results} also apply to stable argumentative ensembling.}

\begin{theorem}
\label{thm:stable_results}
    \JJ{Stable argumentative ensembling is equivalent to d-preferred argumentative ensembling. }
\end{theorem}

\begin{proof}
    \JJ{Let us prove by contradiction that d-preferred extensions are also stable under our BAF $\langle \Args, \Atts, \Supps \rangle$. Let $P^{d}_i \in P^d$ be a d-preferred extension\AR{, which, by definition (Section \ref{sec:preliminaries}), means it is conflict-free}. Assume that $P^{d}_i$ is not stable, i.e., it does not attack all elements that \AR{it does not contain}: $\exists \arg \in \Args$ s.t. $\arg \notin P^{d}_i$, and $P^{d}_i$ does not set-attack $\arg$.
    \AR{By} Definition \ref{def:BAF}, no argument can self-attack, so $(\arg, \arg) \notin \Atts$.} 
    
    \JJ{Let us first consider the case where $\{\arg\}$ set-attacks at least one argument in $P^d_{i}$. Because $P^d_{i}$ is d-preferred, it defends all its elements. Therefore, $P^d_{i}$ must set-attack $\arg$, so we have the contradiction.}
    
    \JJ{Let us consider the case where $\{\arg\}$ does not set-attack any argument in $P^{d}_i$. Because $P^{d}_i$ does not set-attack $\arg$,
    $
    P^{d}_i \cup \{\arg\}
    $ is conflict-free. If $\arg$ has no attacker, \AR{we know that} $P^{d}_i \cup \{\arg\}$ 
    defends all its elements\AR{, and so it must contain $\arg$}. \AR{Thus,} $P^{d}_i$ \AR{cannot} be a d-preferred extension because it is not $\subseteq$-maximal. Then we have the contradiction. Contrastively, let us consider when $\arg$ has one or more attackers, written as $X \subseteq \Args$. \AR{If we consider any $x\in X$, we} 
    need to consider three scenarios depending on the relation between $x$ and $P^{d}_i$. When there exist attacks from $x$ to any argument in $P^{d}_i$, 
    there must also be an attack from $P^{d}_i$ to $x$ (self-defence). Therefore, $P^{d}_i \cup \{\arg\}$ defends all its elements. $P^{d}_i$ \AR{cannot} be a d-preferred extension because it is not $\subseteq$-maximal. Then we have the contradiction. 
    When there are attacks from $P^{d}_i$ to $x$, \AR{the d-preferred extension must} defend all its elements \AR{and so w}e have the same contradiction as the previous case. 
    Finally, when there are no attack relations between $x$ and $P^{d}_i$, 
    $P^{d}_i$ is not \AR{a d-preferred extension because it does not contain the arguments in $X$ and so, again, it is not $\subseteq$-maximal and} we have the contradiction. }

    \JJ{We can therefore conclude that d-preferred extensions are also stable under our BAF definitions and that stable argumentative ensembling is equivalent to d-preferred argumentative ensembling.}
\end{proof}

\subsection{\JJ{S-preferred and c-preferred Argumentative Ensembling}}
\label{ssec:s_extension_theoretical}

We 
now undertake a theoretical analysis of 
\JJ{s-preferred} argumentative ensembling, demonstrating some of the desirable behaviours 
thereof via properties 
introduced in Section \ref{ssec:desirable_properties}.

\begin{theorem}
\label{thm:main}
    \JJ{S-preferred} argumentative ensembling satisfies non-emptiness, model agreement, counterfactual validity and counterfactual coherence. 
    It satisfies non-triviality if for some $\model_i \in \Models$, where $\nexists \model_j \in \Models \setminus \{ \model_i \}$ such that $\model_j \msucc \model_i$, $\exists \model_k \in \Models$ such that $\model_k(\inn) = \model_i(\inn)$, $\model_k(\cfx_i) \!\!\neq\!\! \model_k(\inn)$ and $\model_i(\cfx_k) \!\!\neq\!\! \model_i(\inn)$.
    It does not satisfy majority vote.
\end{theorem}

\begin{proof}
    Let us first prove counterfactual coherence. 
    By Definition~\ref{def:BAF},
    $\forall \model_i \in \Models$ and $\forall \cfx_j \in \CFXs$, $(\model_i, \cfx_j), (\cfx_j, \model_i) \in \Supps$. \JJ{These are the only \AR{support} relations in our BAF.} Thus, there exists an indirect attack on any $\model_i \in \Models$ iff 
    there exists a direct attack on $\cfx_i \in \CFXs$. Likewise, there exists an indirect attack on any $\cfx_i \in \CFXs$ iff  there exists a direct attack on $\model_i \in \Models$. Then, letting $\Models = \{ \model_1, \ldots, \model_m \}$ and $\CFXs = \{ \cfx_1, \ldots, \cfx_m \}$, since we know any $P^s_i \in P^s$ is $\subseteq$-maximal by Definition \ref{def:argensembling}, $P^s_i$ must be such that 
    $\forall i \in \{ 1, \ldots, m \}$, 
    $\model_i \in S$ iff $\cfx_i \in S$.

    Let us prove non-emptiness by contradiction. 
    Assume that \JJ{$\exists P^s_i \in P^s$} such that $P^s_i \cap \Models = \emptyset$ or $P^s_i \cap \CFXs = \emptyset$.
    We know from the above proof that, $\forall P^s_i \in P^s$, $P^s_i \cap \Models \neq \emptyset$ iff $P^s_i \cap \CFXs \neq \emptyset$.
    Then, by the definition of s-preferred extensions (see Section \ref{sec:preliminaries}), it must be the case that $\forall P^s_i \in P^s$, $P^s_i = \emptyset$. 
    Based on the fact that, by Definition
    . \ref{def:modelpref}, $\mpreceq$ is a total ordering and thus transitive, this is not possible as it will always be the case that $\exists \model_i \in \Models$ such that $\nexists \model_j \in \Models$ where $\model_j \msucc \model_i$. Thus, by Definition \ref{def:BAF}, either $\Atts(\model_i) \cup \Atts(\cfx_i) = \emptyset$ or $\forall \arg_k \in \Atts(\model_i) \cup \Atts(\cfx_i)$, $\model_i \in \Atts(\arg_k)$ or $\cfx_i \in \Atts(\arg_k)$, meaning $\{ \model_i, \cfx_i \}$ is either unattacked or is able to defend itself, therefore would be acceptable in at least one s-preferred extension, and we have the contradiction. 

    \JJ{By definition, s-preferred extensions are also d-preferred (see Section \ref{sec:preliminaries}). Therefore, model agreement and counterfactual validity are satisfied (Theorem \ref{thm:d_preferred_results}).}


    Let us prove the \AR{conditional} satisfaction of non-triviality by contradiction. 
    From Definition~\ref{def:argensembling} and the proof for counterfactual coherence above
    , for \JJ{$|S^{a,s} \cap \Models| = 1$}, it must be 
    that $\forall P^s_i \in P^s$, $| P^s_i \cap \Models| = 1$.
    However, from the above assumptions we can see that for some $\model_i \in \Models$, where $\nexists \model_j \in \Models \setminus \{ \model_i \}$ such that $\model_j \msucc \model_i$, $\exists \model_k \in \Models$ such that $\model_k(\inn) = \model_i(\inn)$, $\model_k(\cfx_i) \neq \model_k(\inn)$ and $\model_i(\cfx_k) \neq \model_i(\inn)$.
    Then, to avoid a contradiction, it must be 
    that $\exists \arg_l \in \Models \cup \CFXs$ such that $(\arg_l,\model_k) \in \Atts$ or 
    \AR{$(\arg_l,\cfx_k) \in \Atts$}, and 
    $(\model_i, \arg_l), (\cfx_i, \arg_l) \nin \Atts$.
    \AR{We know that i}f $\arg_l \in \Models$, and $\model_l(\inn) \neq \model_k(\inn) = \model_i(\inn)$, then it \AR{cannot} be the case that \AR{$(\model_l, \model_i) \in \Atts$ and $(\model_i, \model_l) \nin \Atts$} \AR{since,} by 
    Definition~\ref{def:BAF}, $\model_i \msucceq \model_l$.
    \AR{Likewise, i}f $\arg_l \in \CFXs$ and $\model_l(\inn) \neq \model_k(\inn) = \model_i(\inn)$, then \AR{it cannot be the case that $(\model_l, \model_i) \in \Atts$ and $(\model_i, \model_l) \nin \Atts$} by the same reasoning.
    Finally, if $\arg_l \in \CFXs$ and $\model_l(\inn) = \model_k(\inn) = \model_i(\inn)$, then \JJ{it must be} \AR{the case that} $\model_i(\cfx_l) = \model_i(\inn)$ or $\model_l(\cfx_i) = \model_l(\inn)$, then $(\model_i, \cfx_l) \in \Atts$ (a contradiction), and we have the contradiction in all cases.
    

    Finally, Example \ref{ex:arg} provides a counterexample which shows that majority vote is not satisfied.
\end{proof}



\JJ{S-preferred argumentative ensembling 
improves upon augmented, robust, and d-preferred argumentative ensembling by satisfying non-emptiness, counterfactual validity and counterfactual coherence. 
Majority vote is sacrificed in order to achieve this behaviour. In Section \ref{sec:evaluation} we will assess the impact of not guaranteeing majority vote on argumentative ensembling's 
accuracy, along with other metrics. }

\JJ{Next, we show that under our definition of BAF, c-preferred and s-preferred extensions are equivalent. Therefore, the theoretical analysis of s-preferred argumentative ensembling above also applies to c-preferred.}

%
%


\begin{theorem}
\label{thm:c_preferred_results}
    \JJ{C-preferred argumentative ensembling is equivalent to s-preferred argumentative ensembling.}
\end{theorem}

\begin{proof}
    \JJ{By definition, each c-admissible set is also d-admissible (see Section \ref{sec:preliminaries}). S-admissible sets are d-admissible, and we show in Theorem \ref{thm:stable_results} that d-preferred sets are equivalent to stable sets, therefore, by Property 2 of \cite{Cayrol_05}, our s-admissible sets are closed for support relations $\Supps$, therefore are c-admissible.}
\end{proof}

\subsection{\JJ{Discussion on Argumentative Ensembling Behaviours}}
\label{ssec:ensembling_behaviours}

\JJ{In this section, we consider more formal results of argumentative ensembling describing their behaviours. As 
stable and c-preferred argumentative ensemblings are respectively equivalent to their d-preferred and s-preferred counterparts, we will only discuss the latter two onwards. Any results will apply to the former two. For the remainder of the section, we assume as given an input $\inn$, a set $\Models$ of models, a set $\CFXs$ of CEs, a model preference $\mpreceq$ and
a corresponding BAF $\langle \Args, \Atts, \Supps \rangle$.}


\begin{theorem}
\label{thm:nairobarg}
    If $\forall \model_i, \model_j \in \Models$, $\model_i \msimeq \model_j$, and $\forall \model_k \in \Models$ and $\cfx_l \in \CFXs$ where $\model_k(\inn) = \model_l(\inn)$, $\model_k(\cfx_l) \neq \model_k(\inn)$, then \JJ{augmented, robust, and s-preferred argumentative ensembling are equivalent.} 
\end{theorem}

\begin{proof}
    
    If $\forall \model_i, \model_j \in \Models$, $\model_i \msimeq \model_j$, then it can be seen from Definition \ref{def:BAF} that $(\model_i,\model_j), (\model_j,\model_i) \in \Atts$ iff $\model_i(\inn) \neq \model_j(\inn)$. 
    Note that, by Definition \ref{def:baselinecfxs}, augmented and robust ensembling are equivalent since $\forall \model_k \in \Models$ and $\cfx_l \in \CFXs$, where $\model_k(\inn) = \model_l(\inn)$, $\model_k(\cfx_l) \neq \model_k(\cfx_l)$.
    Also by the assumptions, it can be seen from Definition \ref{def:BAF} that $\forall \model_k \in \Models$ and $\forall \cfx_l \in \CFXs$, $(\model_k,\cfx_l), (\cfx_l,\model_k) \in \Atts$ iff $\model_k(\cfx_l) = \model_k(\inn)$ and $\model_k(\inn) \neq \model_l(\inn)$ due to the assumptions in the theorem, meaning any attack is reciprocated and all arguments defend themselves. 
    Then, $\forall \model_m, \model_n \in \Models$ such that $\model_m(\inn) = \model_n(\inn)$, $(\model_m,\model_n) \nin \Atts$. \JJ{Also, any s-preferred extension satisfies counterfactual coherence.} 
    Thus, $P^s = \{ \{ \model_o \in \Models, \cfx_o \in \CFXs | \model_o(\inn) = 0 \} , \{ \model_p \in \Models, \cfx_p \in \CFXs | \model_p(\inn) = 1 \} \AR{\}}$. 
    By Definitions \ref{def:naive_ensembling}, \ref{def:baselinecfxs} and \ref{def:argensembling}, all forms of ensembling select from the same two sets of models and CEs in the same manner and are thus equivalent. 
\end{proof}

\JJ{Note that this result does not hold for d-preferred argumentative ensembling as it does not satisfy non-emptiness. Specifically, its solution can contain only CEs and no models. Due to the same reason, most theoretical findings in this section do not apply to d-preferred argumentative ensembling. We therefore omit relevant discussion unless otherwise mentioned.}

We also provide a number of theoretical 
results concerning the behaviour of 
\AR{s-preferred} argumentative ensembling, first relating to the preferences. 
The first 
result demonstrates how a completely dominant model wrt the preferences will be present in all s-preferred sets.
    
\begin{proposition}
\label{prop:dominant_model}
    If $\exists \model_i \in \Models$ such that $\forall \model_j \in \Models$, $\model_i \msucc \model_j$, then $\forall \JJ{P^{s}_i} \in P^s$, $\model_i \in \JJ{P^{s}_i}$.
\end{proposition}

\begin{proof}
    If $\forall \model_j \in \Models$, $\model_i \msucc \model_j$ then, by Definition \ref{def:BAF}, $\Atts(\model_i) = \Atts(\cfx_i) =\emptyset$ and $\forall \model_j \in \Models$ where $\model_j(\inn) \neq \model_i(\inn)$, $\model_i \in \Atts(\model_j)$. Similarly, $\forall \cfx_k \in \CFXs$  
    where $\model_i(\cfx_k)=\model_i(\inn)$, $\model_i \in \Atts(\cfx_k)$, and so $\model_i$ indirectly attacks $\model_k$. Then, $\nexists \JJ{P^{s}_i} \in P^s$ such that $\model_j \in \JJ{P^{s}_i}$ or $\model_k \in \JJ{P^{s}_i}$ and thus, $\forall \JJ{P^{s}_i} \in P^s$, $\model_i \in \JJ{P^{s}_i}$.
\end{proof}

We also show how, for any two s-preferred sets, there exists some trade-off between their models wrt the preferences.

\begin{lemma}
    Given two s-preferred sets $\JJ{P^{s}_i}, \JJ{P^{s}_j} \in P^s$, $\nexists \model_i \in (\JJ{P^{s}_i} \cap \Models) \setminus \JJ{P^{s}_j}$ such that $\model_i \msucc \model_j$ {for all} $
    \model_j \in \JJ{P^{s}_j}$. 
\end{lemma}


Meanwhile, in the non-strict case, a model which is not 
outperformed by any other model wrt the preferences will be present in at least one s-preferred set.

\begin{proposition}
\label{prop:the_best_model_is_in_m_preferred_set}
    If $\exists \model_i \in \Models$ such that $\forall \model_j \in \Models$, $\model_i \msucceq \model_j$, then $\exists \JJ{P^{s}_i} \in P^s$ such that $\model_i \in \JJ{P^{s}_i}$. 
\end{proposition}

\begin{proof}
    If $\forall \model_j \in \Models$, $\model_i \msucceq \model_j$ then, by Definition \ref{def:BAF}, $\nexists \arg_k \in \Atts(\model_i) \cup \Atts(\cfx_i)$ such that $\model_i \nin \Atts(\arg_k)$ and $\cfx_i \nin \Atts(\arg_k)$. Thus, it must be the case that $\exists \JJ{P^{s}_i} \in P^s$ such that $\model_i \in \JJ{P^{s}_i}$.
\end{proof}

We now show that if a model is outperformed wrt the preferences by all other models, then the outperformed model cannot exist in an s-preferred set unless it is defended by a more preferred model.

\begin{proposition}
\label{prop:worst_model_must_be_defended}
    For any $\model_i \in \Models$, if $\model_j \msucc \model_i$ {for all} $
    \model_j \in \Models \setminus \{ \model_i \}$ and $\exists \JJ{P^{s}_i} \in P^s$ such that $\model_i \in \JJ{P^{s}_i}$, then, $\forall \model_k \in \Models$ where $\model_k(\inn) \neq \model_i(\inn)$,  
    $\exists \model_l \in \JJ{P^{s}_i}$ such that $\model_l \msucceq \model_k$.
\end{proposition}

\begin{proof}
    By Definition \ref{def:BAF}, $\model_k \in \Atts(\model_i)$ and $\{ \model_m \in \Models | \model_i \in \Atts (\model_m) \vee \cfx_i \in \Atts (\model_m) \} = \emptyset$. 
    Then, given that $\exists \JJ{P^{s}_i} \!\in\! P^s$ such that $\model_i \!\in\! \JJ{P^{s}_i}$, we know that, $\forall \model_k \!\in\! \Atts(\model_i)$ $\exists \model_l \!\in\! X$ where $\model_l(\inn) \!=\! \model_i(\inn)$ and $\model_l \!\in\! \Atts(\model_k)$. Thus, by Definition \ref{def:BAF}, $\model_l \msucceq \model_k \msucc \model_i$.
\end{proof}

\JJ{Propositions \ref{prop:the_best_model_is_in_m_preferred_set} and \ref{prop:worst_model_must_be_defended} also apply to d-preferred \AR{argumentative ensembling} with the same proof.}

We 
also consider the behaviour of 
argumentative ensembling wrt the selected CEs, demonstrating 
that those from disagreeing models are guaranteed not to be included in any s-preferred set.

\begin{proposition}
    Any s-preferred set $\JJ{P^{s}_i} \in P^s$ is such that $\nexists \cfx_i, \cfx_j \in \JJ{P^{s}_i} \cap \CFXs$ where $\model_i(\inn) \neq \model_j(\inn)$.
\end{proposition}
\begin{proof}
    Let us prove by contradiction, assuming that $\exists \cfx_i, \cfx_j \in \JJ{P^{s}_i} \cap \CFXs$.
    Counterfactual coherence (Theorem \ref{thm:main}) requires that $\model_i, \model_j \in \JJ{P^{s}_i}$. 
    However, by Definition \ref{def:BAF}, $\model_i(\inn) \neq \model_j(\inn)$ requires that $\model_i \in \Atts(\model_j)$ or  $\model_j \in \Atts(\model_i)$, and so we have the contradiction.
\end{proof}




\subsection{\JJ{Additional Discussion on Abstract vs. Bipolar Argumentation}}
\label{ssec:aaf}

\JJ{In this section, we discuss the similarities and differences between our proposed argumentative ensembling approach instantiated using bipolar argumentation and an alternative approach instantiated using abstract argumentation. In our previous definition of a CE (Section \ref{sec:preliminaries}), we required \AR{that} $\model(\cfx) \neq \model(\inn)$, i.e., the CE for a model prediction is \emph{valid} on the model. Under this assumption, we obtain the following equivalence result:}

\begin{theorem}
\label{thm:abstract_argumentation_equivalence}
\JJ{Consider input x, set $\Models$ of models, set $\CFXs$ of CEs such that $\forall \cfx_i \in \CFXs$, $\forall \model_i \in \Models$, $\model_i(\cfx_i) \neq \model_i(\inn)$.
Let $\mpreceq\!$ be the model preferences, and 
$\langle \Args, \Atts, \Supps \rangle$ be the corresponding BAF with $P^s$ the set of all s-preferred extensions.
Let $\langle \Args', \Atts' \rangle$ be an AAF with:}
    \begin{itemize}
        \item \JJ{$\Args' = 
        \Models \times \CFXs$;}
        \item $\Atts' \subseteq \Args' \times \Args'$ where
            $\forall (\model_i, \cfx_i), (\model_j, \cfx_j) \in \Args', ((\model_i, \cfx_i), (\model_j, \cfx_j)) \in \Atts'$ iff $\model_i \!\msucceq \!\model_j$ and:
            \begin{itemize}
                \item $\model_i(\inn) \neq \model_j(\inn)$; or
                \item $\model_i(\cfx_j) \!= \!\model_i(\inn)$; or
                \item $\model_j(\cfx_i) \!= \!\model_j(\inn)$.
            \end{itemize}
        \end{itemize}
Let $P$ be the set of all preferred extensions of $\langle \Args', \Atts' \rangle$. 
Then, 
\begin{itemize}
    \item $P^s=\{f(P_i)| P_i \in P \}$
    \item $P=\{f^{-1}(P^s_i)| P^s_i \in P^s \}$
\end{itemize}

\AR{where} $f$ \AR{is} a function such that $f(P_i) = \{ \model_i, \cfx_i | (\model_i, \cfx_i) \in P_i \}$

\AR{and} $f^{-1}$ \AR{is} such that $f^{-1}(P^s_i) = \{ (\model_i, \cfx_j) \in \Models \times \CFXs | \model_i, \cfx_j \in P^s_i, i = j \}$.





\end{theorem}

\begin{proof}

\JJ{First, we prove $\forall P^s_i \in P^s$, $f^{-1}(P^s_i) \in P$. 
We show $f^{-1}(P^s_i)$ is conflict-free by contradiction. Assume $f^{-1}(P^s_i)$ is not conflict-free, then $\exists (\model_i, \cfx_i), (\model_j, \cfx_j) \in f^{-1}(P^s_i)$, such that $((\model_i, \cfx_i), (\model_j, \cfx_j)) \in \Atts'$. Then, by the AAF definition in Theorem \ref{thm:abstract_argumentation_equivalence}, it must be that $\model_i \!\msucceq \!\model_j$ and, either $\model_i(\inn) \neq \model_j(\inn)$, or $\model_i(\cfx_j) \!= \!\model_i(\inn)$, or $\model_j(\cfx_i) \!= \!\model_j(\inn)$. By Definition \ref{def:BAF}, in the first case, $(\model_i, \model_j) \in \Atts$; in the second and third cases, $(\model_j, \cfx_j) \in \Atts$. Then, $P^s_i$ cannot be conflict-free, but it is an s-preferred extension, so we have the contradiction. 
We now prove $f^{-1}(P^s_i)$ defends all its elements 
by contradiction. Assume $f^{-1}(P^s_i)$ does not defend all its elements, then, $\exists(\model_i, \cfx_i) \in f^{-1}(P^s_i)$, such that $\exists (\model_k, \cfx_k) \in \Args'\setminus f^{-1}(P^s_i)$ \AR{where} $((\model_k, \cfx_k), (\model_i, \cfx_i)) \in \Atts'$, and $\nexists (\model_j, \cfx_j) \in f^{-1}(P^s_i)$ s.t. $((\model_j, \cfx_j), (\model_k, \cfx_k)) \in \Atts'$. Mapping this back to the BAF by Definition \ref{def:BAF}, we know $\model_i, \cfx_i \in P^s_i$, $\model_k, \cfx_k \in \Args \setminus P^s_i$, and either $(\model_k$, $\model_i) \in \Atts$, $(\model_k, \cfx_i) \in \Atts$, or $(\cfx_k, \model_i) \in \Atts$, and $\nexists \model_j, \cfx_j \in P^s_i$ s.t. $(\model_j$, $\model_k) \in \Atts$, $(\model_j, \cfx_k) \in \Atts$, or $(\cfx_j, \model_k) \in \Atts$. 
This cannot be the case since $P^s_i$ is safe 
and so we have the contradiction. Therefore, $f^{-1}(P^s_i)$ is conflict-free and it defends all its elements in AAF, i.e., it is a preferred extension in the AAF.}

\JJ{Next, we prove $\forall P_i \in P, f(P_i) \in P^s$. 
We first show $f(P_i)$ is safe by contradiction. Assume $f(P_i)$ is not safe, then $\exists \arga \in \Args, f(P_i)$ set-attacks $\arga$, and either $\arga \in f(P_i)$, or $f(P_i)$ set-supports $\arga$. For the former case, we know $P_i$ is conflict-free in \AR{the} AAF. Mapping the conditions of attacks in \AR{the} AAF (Theorem \ref{thm:abstract_argumentation_equivalence} above) to the conditions of attacks in BAF (Definition \ref{def:BAF}), it means $\nexists \model_i, \model_j \in f(P_i)$ s.t. $(\model_i, \model_j) \in \Atts$. Also, $\nexists \model_i, \cfx_j \in f(P_i)$ s.t. $(\model_i, \cfx_j) \in \Atts$ or $(\cfx_j, \model_i) \in \Atts$. Therefore, we have a contradiction in the former case. \AR{There is also a contradiction in} the second case, because by Definition \ref{def:BAF}, support relations can only exist between a model and its CE. Therefore, $f(P_i)$ is safe. Next, we prove $f(P_i)$ defends all its elements. We know that $P_i$ defends all its elements in the AAF because it is a preferred extension. This means $\forall (\model_k, \cfx_k) \in \Args\setminus P_i$, if $(\model_k, \cfx_k)$ attacks an element in $P_i$, then $\exists (\model_i, \cfx_i) \in P_i$ s.t. $((\model_i, \cfx_i), (\model_k, \cfx_k)) \in \Atts'$. Translating the conditions of attacks in the AAF to the conditions of attacks in the BAF, this indicates that in the BAF, one of $(\model_k, \cfx_k)$ must be directly attacked and the other one is supported attacked, because they support each other. Therefore, any attacker of $f(P_i)$ is also attacked by $f(P_i)$ in the BAF. Therefore, $f(P_i)$ is an s-preferred extension in the BAF because it defends all its elements is safe.}
\end{proof}

\begin{example}
\label{ex:AAF}
        \JJ{The AAF corresponding to input, models and CEs as in Example \ref{
        ex:CEs} is $\langle \Args', \Atts'\rangle$ with (see  Fig.~\ref{fig:example_AAF}):
    $\Args' = $
    $\{ (\model_1, \cfx_1),$ 
    $(\model_2, \cfx_2),$
    $(\model_3, \cfx_3),$
    $(\model_4, \cfx_4),$
    $(\model_5, \cfx_5)\}$; 
    $\Atts' = $ 
    $\{((\model_2, \cfx_2), $ 
    $(\model_1, \cfx_1)), $ 
    $ ((\model_2, \cfx_2), $ 
    $ (\model_3, \cfx_3)), $ 
    $ ((\model_2, \cfx_2), $ 
    $ (\model_4, \cfx_4)), $ 
    $ ((\model_2, \cfx_2), $ 
    $ (\model_5, \cfx_5)), $ 
    $ ((\model_3, \cfx_3), $ 
    $ (\model_4, \cfx_4)), $ 
    $ ((\model_4, \cfx_4), $ 
    $ (\model_1, \cfx_1)), $ 
    $ ((\model_5, \cfx_5), $ 
    $ (\model_1, \cfx_1)), $ 
    $ ((\model_5, \cfx_5), $ 
    $ (\model_2, \cfx_2)), $ 
    $ ((\model_5, \cfx_5), $ 
    $ (\model_3, \cfx_3))\}$,  
    this leads to 
    $P = $
    $\{ \{ (\model_2,$ 
    $\cfx_2) \},$ 
    $\{ (\model_4,$ 
    $\cfx_4),$ 
    $(\model_5,$ 
    $\cfx_5) \} \} $, and the cardinality-maximal preferred extension, which becomes the solution to the RAE problem, is 
    $\{ (\model_4,$ 
    $\cfx_4),$ 
    $(\model_5,$ 
    $\cfx_5) \} $,  \JJ{the aggregated prediction is} $\model_4(\inn) = \model_5(\inn) = 1$.}
\end{example}

\begin{figure}[h]
    \centering
    \includegraphics[width=0.66\textwidth]{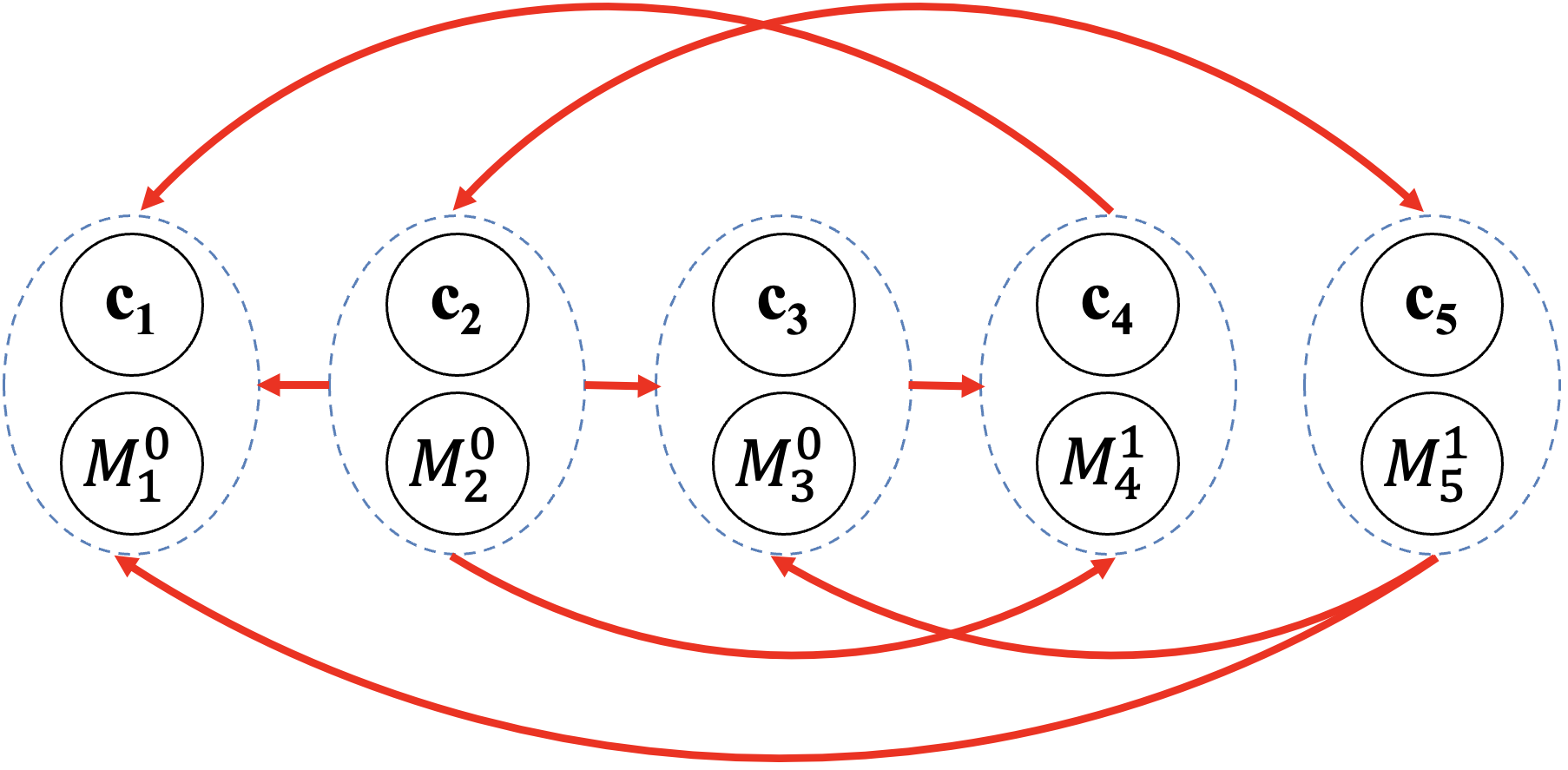}
	\caption{\JJ{AAF for Example \ref{ex:AAF} where: models' predictions for the input $\inn$ are given as superscripts, e.g., $\model_1(\inn) = 0$ \AR{and} $\model_4(\inn) = 1$; standard (reciprocal) attacks are represented by single-headed (\AR{double}-headed, \resp) red arrows. }
    \label{fig:example_AAF}}
\end{figure}

\JJ{Example \ref{ex:AAF} (Figure \ref{fig:example_AAF}), the AAF equivalent of Example \ref{ex:arg} (Figure \ref{fig:example_BAF}), demonstrates the intuitions behind the equivalence between BAF and AAF.}

\JJ{In practice, however, CEs may not always be valid on the model they are generated for. This is especially the case for gradient-based methods because no validity guarantees are given \cite{Wachter_17,guidotti2022counterfactual}. When this happens, grouping each model and the corresponding CE can become problematic because of the invalidity. Then, \AR{the} BAF provides a more expressive way to represent the conflicts than \AR{the} AAF. For example, the support relations between a model and its CE can be removed or changed into attack relations. We leave the investigation of such scenarios to future work.}

\section{Empirical Evaluation}
\label{sec:evaluation}





We now examine the effectiveness of our approach using three real-world datasets. Specifically, we empirically 
evaluate the extent to which each of the ensembling 
methods introduced in Section \ref{ssec:naive} and Section \ref{sec:main} satisfy the desirable properties defined in Section \ref{ssec:desirable_properties}. \JJ{We instantiate four variations of 
argumentative ensembling with different model preferences over two types of model properties 
$\Props$ and demonstrate the usefulness of incorporating model preferences. We also show that the task performance (accuracy) by our methods remain comparable to the baselines. Finally, we report the compute time for solving BAFs against increasing model sizes in the problem setting.}


\subsection{Experiment Setup} 

\subsubsection{\JJ{Datasets and Models}}
We apply all ensembling methods on three datasets in the legal and financial contexts: loan approval (heloc) \cite{heloc}, recidivism prediction (compas) \cite{compas}, and credit risk (credit) \cite{Dua2019}. Due to 
neural networks' sensitivity to randomness at training time, 
they suffer severely from MM and are frequently targeted when investigating this research topic (as discussed in Section \ref{sec:related}). \JJ{Therefore, we focus on neural networks for our experiments. Note that the problem setting focuses on the conflicts between classifiers' predictions and CE validities and is agnostic to the types of models used. Our method is also model agnostic and can be applied to model sets consisting of even different types of classifiers.}

\begin{table}[h]
\centering
\begin{tabular}{cccccc}
\hline
 & (15, 15) & (20, 15) & (20, 20) & (30, 20) & (30, 25) \\ \hline
heloc & .730$\pm$.015 & .738$\pm$.009 & .741$\pm$.008 & .734$\pm$.007 & .736$\pm$.009 \\
compas & .844$\pm$.013 & .841$\pm$.013& .845$\pm$.011& .840$\pm$.012& .845$\pm$.012 \\ 
credit & .709$\pm$.028 & .699$\pm$.007& .705$\pm$.018& .697$\pm$.011& .696$\pm$.021 \\ \hline
\end{tabular}
\caption{\JJ{Accuracies of trained neural networks of different sizes on each dataset}}
\label{tab:model_accuracies}
\end{table}

For each dataset, we train-test 150 classifiers with five different hidden layer sizes using 80\% of the dataset (this 80\% is train-test split for training each model). \JJ{The neural networks \AR{have} two hidden layers of sizes \{(15, 15), (20, 15), (20, 20), (30, 20), (30, 25)\}. The five different sizes are for obtaining five levels of model simplicity, which serves as a toy property evaluation to demonstrate the usefulness of supporting model preferences in our framework. Their 5-fold cross-validation accuracies on the training set are reported in Table \ref{tab:model_accuracies}. We observe no significant changes in test accuracies across the different model sizes, as the accuracy variations fall within the range of standard deviation.} 
\AR{The classifiers} are trained using different random seeds for parameter initialisation and 
different train-test splits (within the train-test 80\% of the dataset), forming a pool of possible models under MM from which we sample multiple sets  $\Models$ of models to which we apply our ensembling methods. 
We use the remaining 20\% of each dataset as test inputs for the ensembling methods (limited to 500 inputs test set if larger).

\subsubsection{\JJ{Experimental Procedure}}

\JJ{We instantiate the \Problem{} problem by randomly sampling sets $\Models$ with 10, 20 or 30 models from the model pool of the 150 trained classifiers.} 
We then  
feed each input to the models to receive their predicted labels and generate one CE from each model using the nearest neighbour CEs approach of \cite{NiceNNCE}, \JJ{due to its effectiveness and guaranteed CE validity}. We finally apply the ensembling methods. For each size $|\Models| = (10, 20, 30)$, we \JJ{ repeat the above process five times and record the average 
 results.}

As concerns model preferences, we 
focus on \JJ{two measurable model properties:} 
accuracy of the trained classifiers over the (20\%) test inputs and model structure simplicity. \JJ{For accuracy, we manually made sure that every model in $\Models$ has a different test accuracy and thus no pairs of models can make identical predictions to every test input.} 
For the latter, we assign the models, from the most complex to the simplest (depending on the number of neurons in the hidden layers\AR{)}, scores of \{0, 0.25, 0.5, 0.75, 1\} such that higher values imply simpler models. Note that 
multiple models in $\Models$ may have the same simplicity scores 
as we adopt only five different model structures to obtain 
150 neural networks for each dataset. \JJ{Using the two metrics, we specify four types of model preferences to be used in our method, discussed more later in Section  \ref{sec:experiment_setup_ensembling_methods}.}


\subsubsection{\JJ{Evaluation Metrics}}

\JJ{Each ensembling method is evaluated against the following metrics: aggregated prediction accuracy over the test set, average model simplicity in the ensemble, and the percentage of test points for which each property is satisfied. 
We also report the average test set accuracies of the models in $\Models$. Finally, we report the average computation time needed for argumentative ensembling when the BAF is built to show the scalability of the method.}

\subsubsection{\JJ{Ensembling Methods}}
\label{sec:experiment_setup_ensembling_methods}

We use augmented ensembling $S^n$ 
and robust ensembling $S^v$ as 
baselines. \JJ{We experiment with both d-preferred and s-preferred argumentative ensembling ($S^{a,d}$ and $S^{a,s}$ respectively).  
For 
each type of} argumentative ensembling, we use four variations with different
preferences
: $S^{a, \sigma}$ ($\Props\!=\!\emptyset$), $S^{a, \sigma}$-A ($\Props\!=\!\{accuracy\}$), $S^{a, \sigma}$-S ($\Props\!=\!\{simplicity\}$) and $S^{a, \sigma}$-AS ($\Props\!=\!\{accuracy$, $simplicity\}$, 
{with} $accuracy \!\psimeq\! simplicity$). 

\JJ{In our implementation of argumentative ensembling, we add simple mechanisms for handling edge cases. Firstly, when there are multiple cardinality-maximum extensions, we return the one that gives the same prediction as majority vote. If there are still multiple sets, we then randomly choose one. Additionally, for d-preferred argumentative ensembling, when there are multiple cardinality-maximum extensions, where possible, we select from the ones that contain at least one model and one CE. The implementations are adapted from \cite{DBLP:conf/atal/JiangL0T24,jiang2025robustx}.}

\subsection{\JJ{Results on the Usefulness of Preferences}}

\begin{table}[h!]
\centering
\resizebox{0.75\textwidth}{!}{
\begin{tabular}{cl|cc|cc|cc} 
\hline
\textbf{Dataset} & \textbf{Method} &  \multicolumn{2}{c|}{\textbf{$|\Models|=10$}} & \multicolumn{2}{c|}{\textbf{$|\Models|=20$}} & \multicolumn{2}{c}{\textbf{$|\Models|=30$}} \\

 & & \textbf{Acc.} & \textbf{Simp.} & \textbf{Acc.} & \textbf{Simp.}& \textbf{Acc.} & \textbf{Simp.}\\ \hline

\multirow{11}{*}{heloc} & avg & .710 & .485 & .710 & .500 & .710 & .510 \\
                          & $S^n$ & .714 & .503 & .716 & .508 & .716 & .508\\
                          & $S^v$ & .714 & .503& .716 & .508& .716&  .508\\
                          & $S^{a,d}$ & .714 & .507& .709 & .505 & .710& .501\\
                          & $S^{a,d}-A$ & .720 & .478& .726 & .547 & .728 & .507\\
                          & $S^{a,d}-S$ & .710 &.606 & .712&.542 & .711 & .520\\
                          & $S^{a,d}-AS$ & .716 &.529 & .711&.520 & .709 & .506\\
                          & $S^{a,s}$ & .714 &.511 & .709 &.500 & .710 & .502\\
                          & $S^{a,s}-A$ & .720 &.478 & .726 &.547 & .728& .507\\
                          & $S^{a,s}-S$ & .710 &.607 & .712 & .541& .711& .520\\
                          & $S^{a,s}-AS$ & .716 & .529& .711 & .520& .709& .507\\ \hline
\multirow{11}{*}{compas} & avg & .855 & .570 & .855 & .558 & .855 & .533 \\
                          & $S^n$ & .856 & .495 & .858 & .490 & .859 & .493 \\
                          & $S^v$ & .856 & .495& .858& .490 & .859 & .493\\
                          & $S^{a,d}$ & .860 &.492 & .858 & .479 & .858& .480\\
                          & $S^{a,d}-A$ & .863 &.436 & .866& .619 & .869 & .583\\
                          & $S^{a,d}-S$ & .861& .699 & .857& .586 & .856 & .551 \\
                          & $S^{a,d}-AS$ & .864 & .567 & .856& .519 & .856 & .518 \\
                          & $S^{a,s}$ & .860 & .492 & .858 & .476 & .858 & .479\\
                          & $S^{a,s}-A$ & .863 & .436 & .866 & .619 & .869& .583\\
                          & $S^{a,s}-S$ & .861 & .698& .857 & .587 & .856& .545\\
                          & $S^{a,s}-AS$ & .864 &.585 & .856 & .529& .856& .532 \\ \hline
\multirow{11}{*}{credit} & avg & .660 & .505 & .660 & .510 & .660 & .520 \\
                          & $S^n$ & .696 & .509 & .705 & .533 & .709 & .545 \\
                          & $S^v$ & .696 & .509 & .705 & .533 & .709 & .545\\
                          & $S^{a,d}$ & .685 & .525 & .683 & .556 & .682 & .563\\
                          & $S^{a,d}-A$ & .700 & .457 & .706 & .522 & .710 & .554\\
                          & $S^{a,d}-S$ & .670 & .610 & .686 & .564& .685 & .568\\
                          & $S^{a,d}-AS$ & .683 & .537 & .685 & .559& .684 & .566\\
                          & $S^{a,s}$ & .685 & .522 & .683 & .555 & .682 & .563\\
                          & $S^{a,s}-A$ & .700 & .437 & .706& .521 & .710& .552\\
                          & $S^{a,s}-S$ & .670 & .611 & .686 & .563 & .685& .568\\
                          & $S^{a,s}-AS$ & .683 & .535 & .685 & .559 & .684 & .566\\ \hline
\end{tabular}
}
\caption{Average accuracy and simplicity of the models in the ensembling result by each method. ``avg'' rows report these scores for the models in $\Models$.}
\label{tab:result_accuracy}
\end{table}


\JJ{We first report the accuracy and average model simplicity of the ensembles by each method in Table \ref{tab:result_accuracy} (note that these are obtained on subsets different than the ones we used to obtain cross-validation accuracies in Table \ref{tab:model_accuracies}).} With test accuracy specified as model preference, $S^{a,\sigma}$-A shows the best accuracy in all experiments (across d- and s-preferred). This validates Proposition \ref{prop:dominant_model}, because, assuming that the accuracy for every model in $\Models$ is different, for $S^{a,\sigma}$-A, there exists a model in 
$\Models$ that is the most preferred 
and is included in the ensemble. Similarly, $S^{a,\sigma}$-S shows the best simp. scores in all experiments. However, since simplicity scores are not unique for each model, usually a single most preferred model does not exist, therefore an optimal simp. evaluation is not guaranteed. 
When specifying both properties as model preferences ($S^{a,\sigma}$-AS), at least one of the two metrics is improved compared with $S^{a,\sigma}$.

\JJ{It can be concluded that any ensembling method improves prediction accuracy when compared with the average single-model accuracies, as better evaluations can almost always be observed. While $S^{a,\sigma}-A$ configurations always give the highest accuracy, the majority vote accuracy remains competitive. For four out of nine sets of experiments (heloc $|\Models|=10$, compas $|\Models|=10, 20, 30$), argumentative ensembling accuracies when no preferences are specified are higher or comparable with $S^n$ and $S^v$, while they are marginally lower in other cases.}

\subsection{\JJ{Results on Desirable Properties Satisfaction}}


\begin{figure}[h!]
     \centering
     \begin{subfigure}{1\textwidth}
         \centering
         \includegraphics[width=\textwidth]{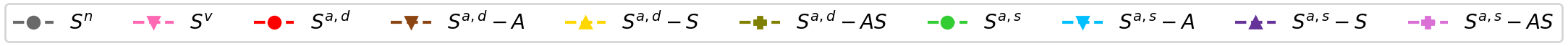}
     \end{subfigure}
     \begin{subfigure}{1\textwidth}
         \centering\includegraphics[width=0.8\textwidth]{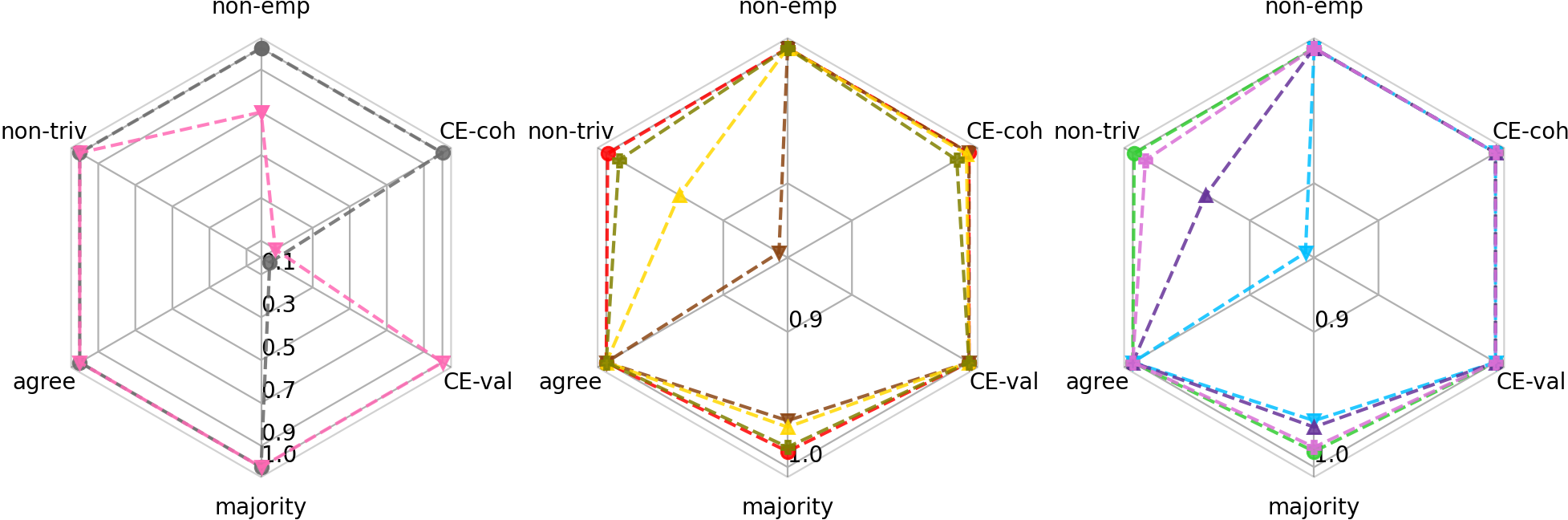}
         \caption{heloc dataset, $|\Models|=10$}
         \label{fig:vis1_heloc}
     \end{subfigure}
     \begin{subfigure}{1\textwidth}
         \centering
         \includegraphics[width=0.8\textwidth]{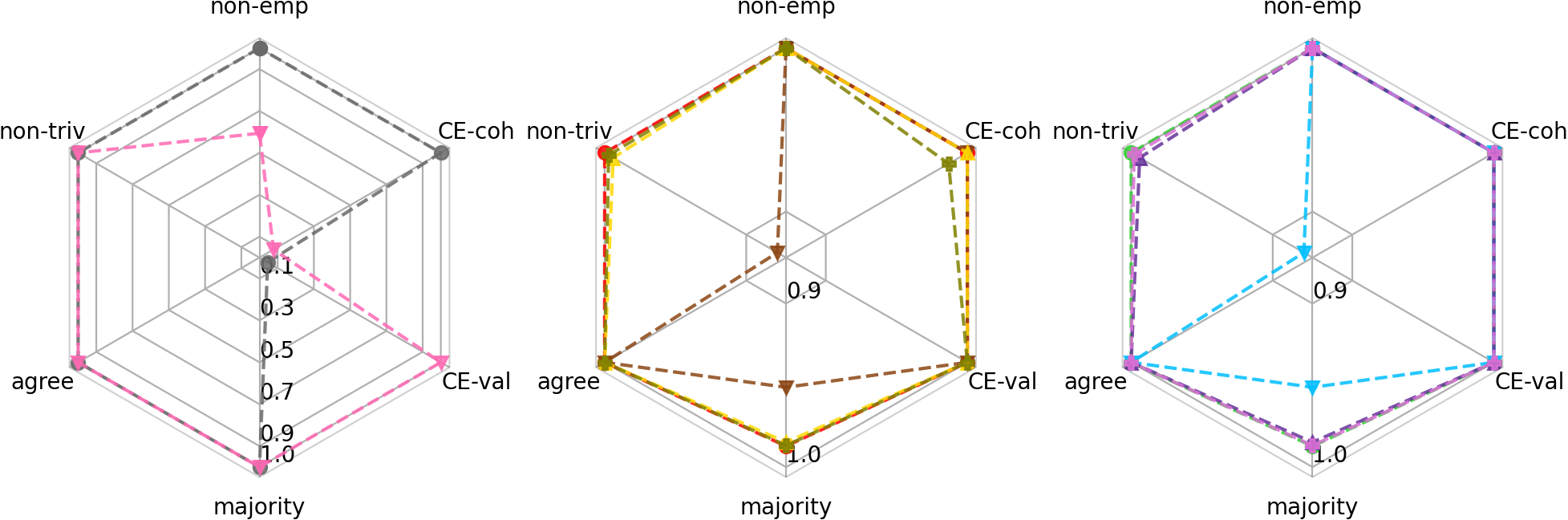}
         \caption{heloc dataset, $|\Models|=20$}
         \label{fig:vis1_heloc}
     \end{subfigure}
     \hfill
     \begin{subfigure}{1\textwidth}
         \centering
         \includegraphics[width=0.8\textwidth]{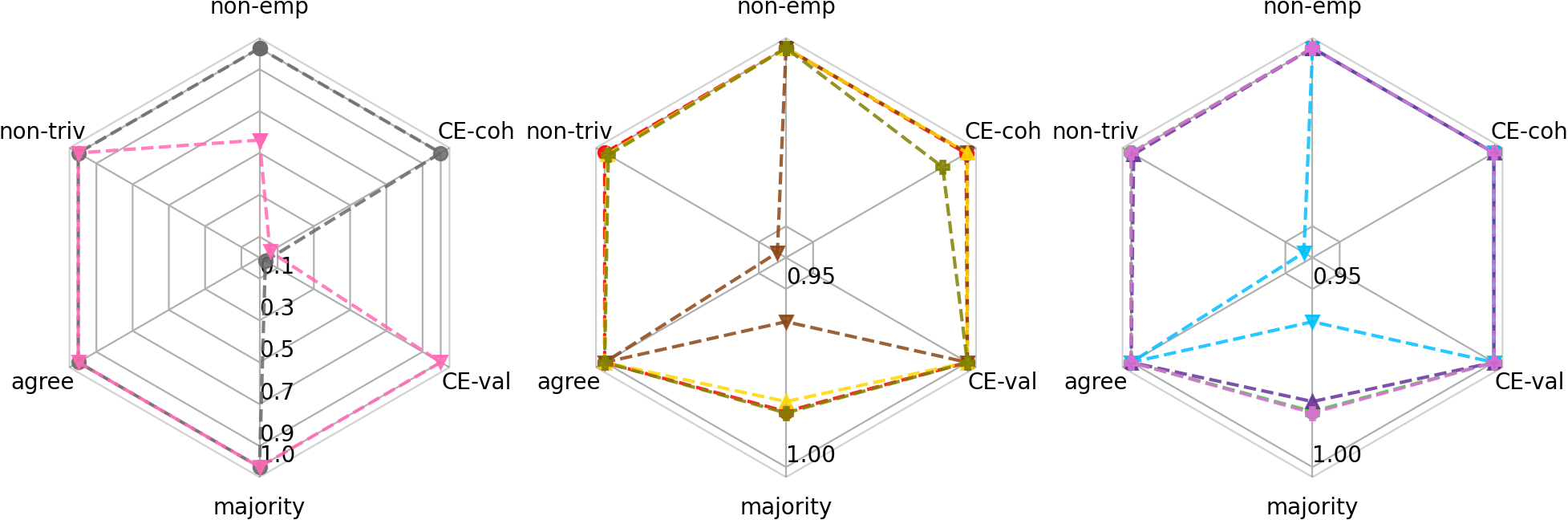}
         \caption{heloc dataset, $|\Models|=30$}
         \label{fig:vis3_heloc}
     \end{subfigure}
        \caption{\JJ{Satisfaction of desirable properties on heloc dataset. The three subplots show results respectively for the two baselines, d-preferred argumentative ensembling, and s-preferred argumentative ensembling.}}
        \label{fig:property_satisfaction_heloc}
\end{figure}

\begin{figure}[h!]
     \centering
     \begin{subfigure}{1\textwidth}
         \centering
         \includegraphics[width=\textwidth]{imgs/legend_only.png}
     \end{subfigure}
     \begin{subfigure}{1\textwidth}
         \centering\includegraphics[width=0.8\textwidth]{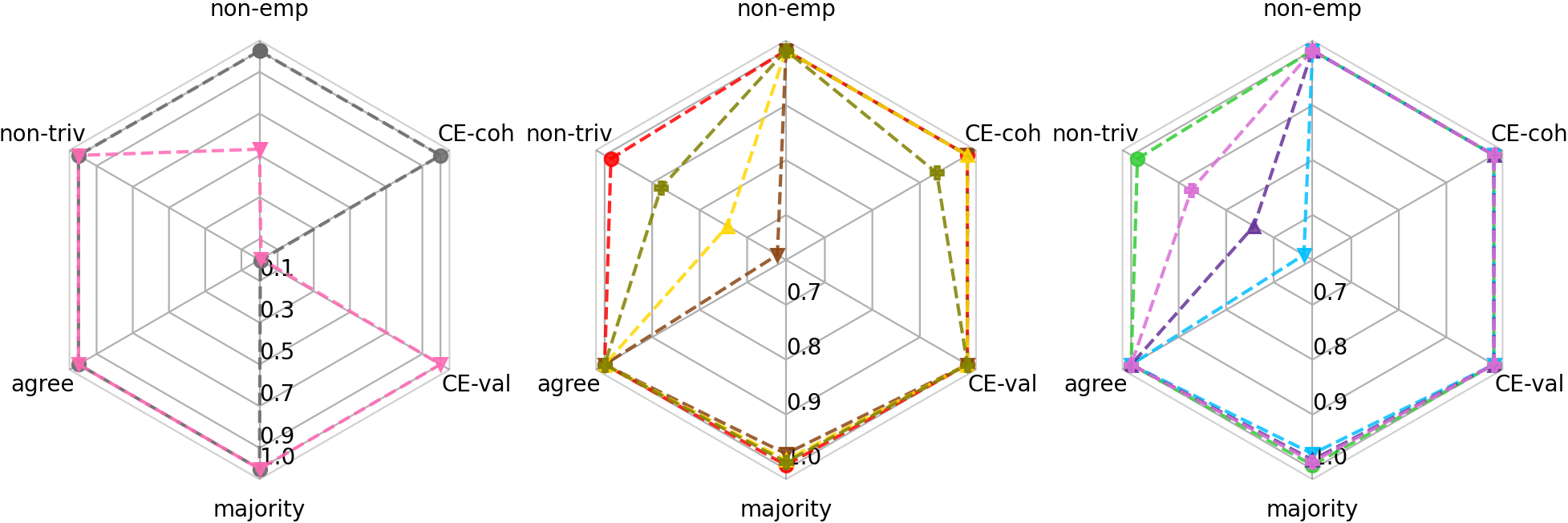}
         \caption{compas dataset, $|\Models|=10$}
         \label{fig:vis1_compas}
     \end{subfigure}
     \begin{subfigure}{1\textwidth}
         \centering
         \includegraphics[width=0.8\textwidth]{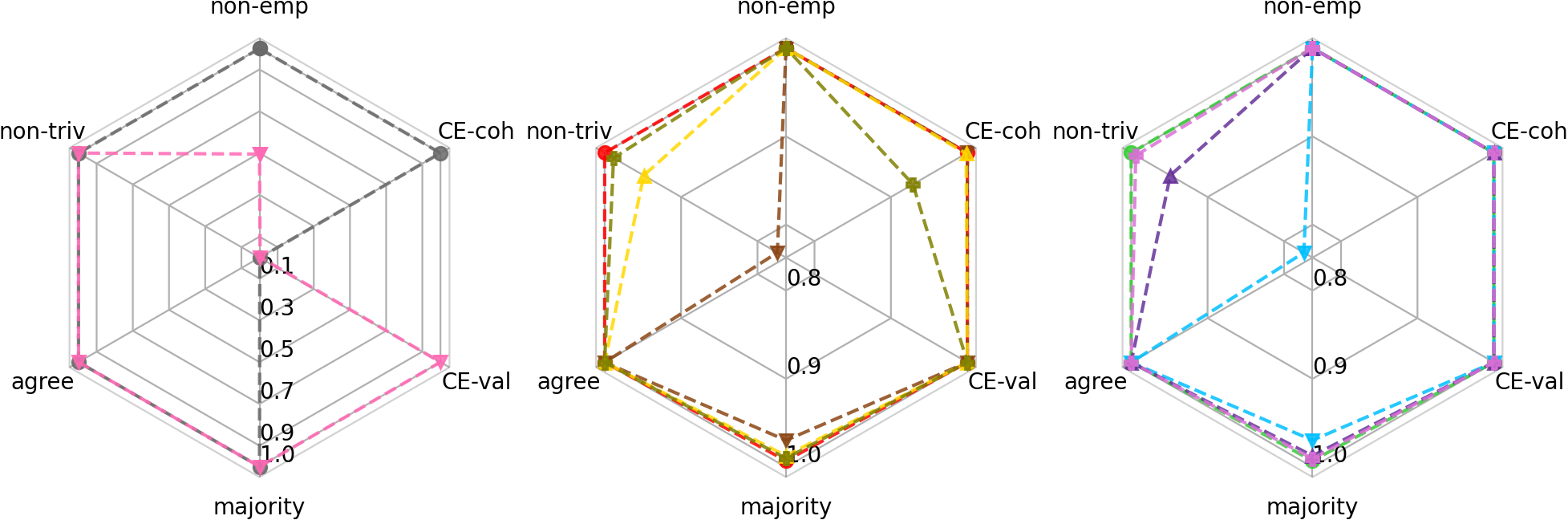}
         \caption{compas dataset, $|\Models|=20$}
         \label{fig:vis2_compas}
     \end{subfigure}
     \hfill
     \begin{subfigure}{0.8\textwidth}
         \centering
         \includegraphics[width=\textwidth]{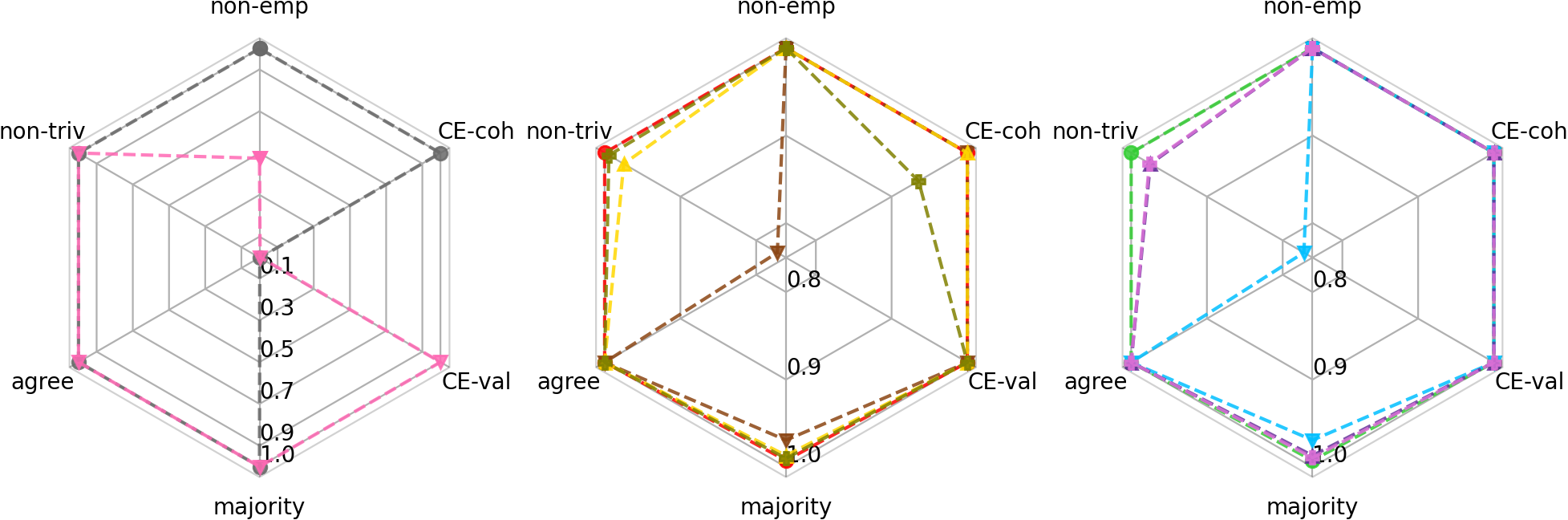}
         \caption{compas dataset, $|\Models|=30$}
         \label{fig:vis3_compas}
     \end{subfigure}
        \caption{\JJ{Satisfaction of desirable properties on compas dataset. The three subplots show results respectively for the two baselines, d-preferred argumentative ensembling, and s-preferred argumentative ensembling.}}
        \label{fig:property_satisfaction_compas}
\end{figure}

\begin{figure}[h!]
     \centering
     \begin{subfigure}{1\textwidth}
         \centering
         \includegraphics[width=\textwidth]{imgs/legend_only.png}
     \end{subfigure}
     \begin{subfigure}{1\textwidth}
         \centering\includegraphics[width=0.8\textwidth]{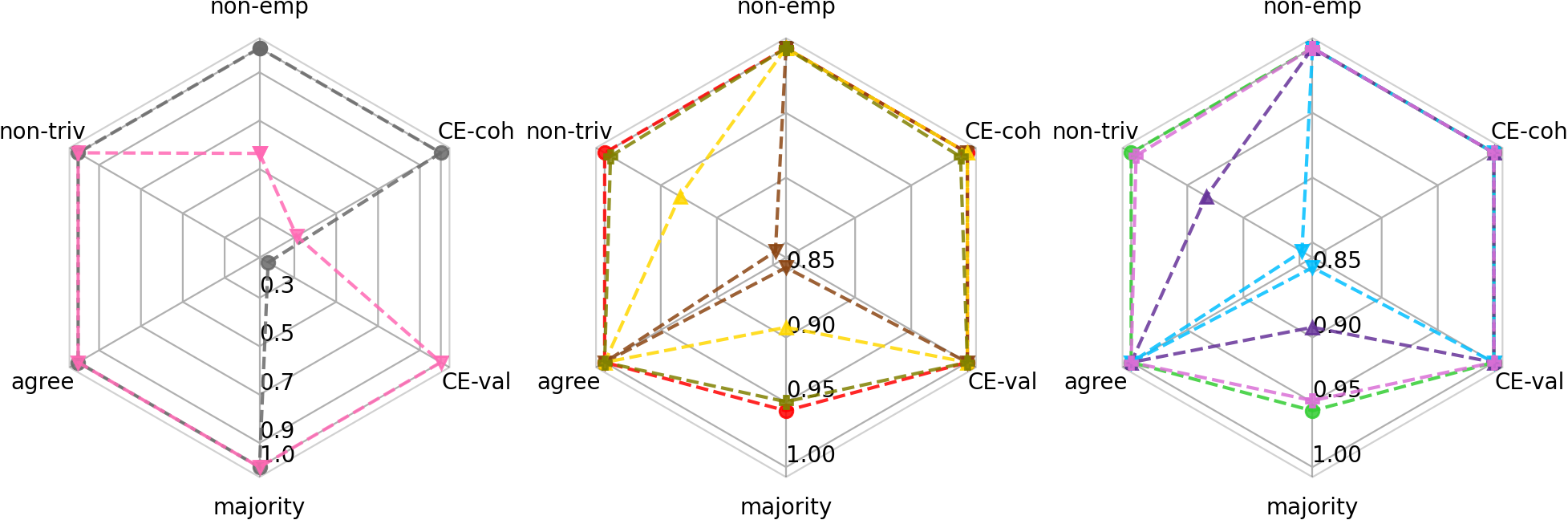}
         \caption{credit dataset, $|\Models|=10$}
         \label{fig:vis1_credit}
     \end{subfigure}
     \begin{subfigure}{1\textwidth}
         \centering
         \includegraphics[width=0.8\textwidth]{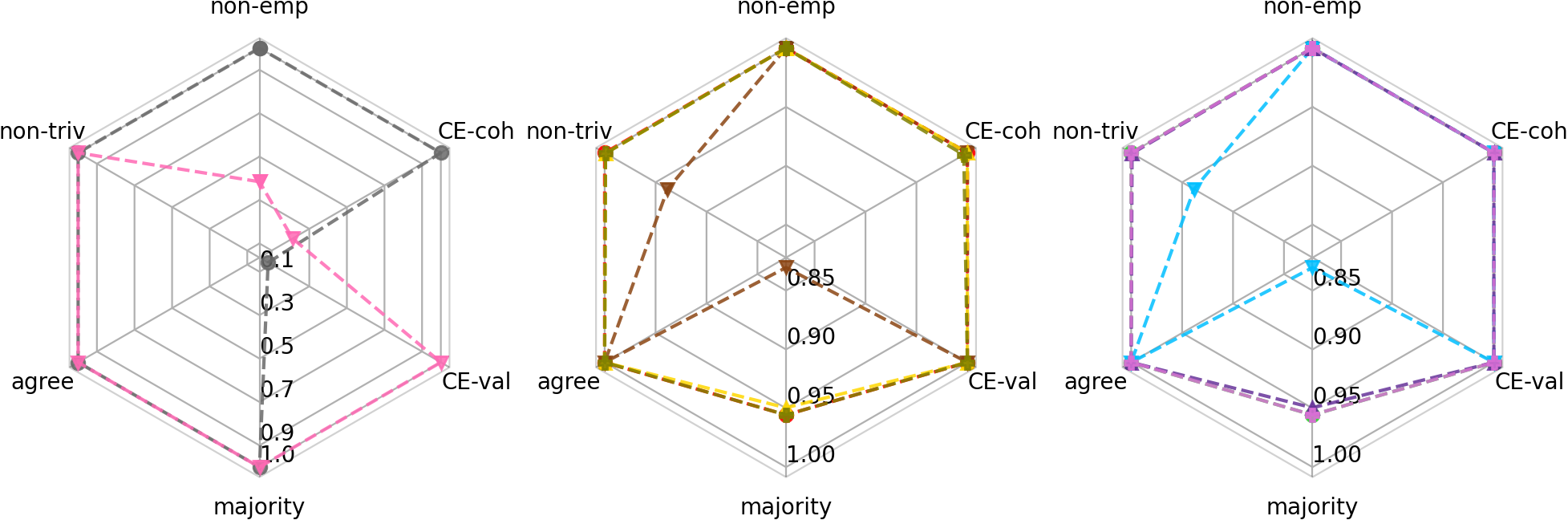}
         \caption{credit dataset, $|\Models|=20$}
         \label{fig:vis2_credit}
     \end{subfigure}
     \hfill
     \begin{subfigure}{1\textwidth}
         \centering
         \includegraphics[width=0.8\textwidth]{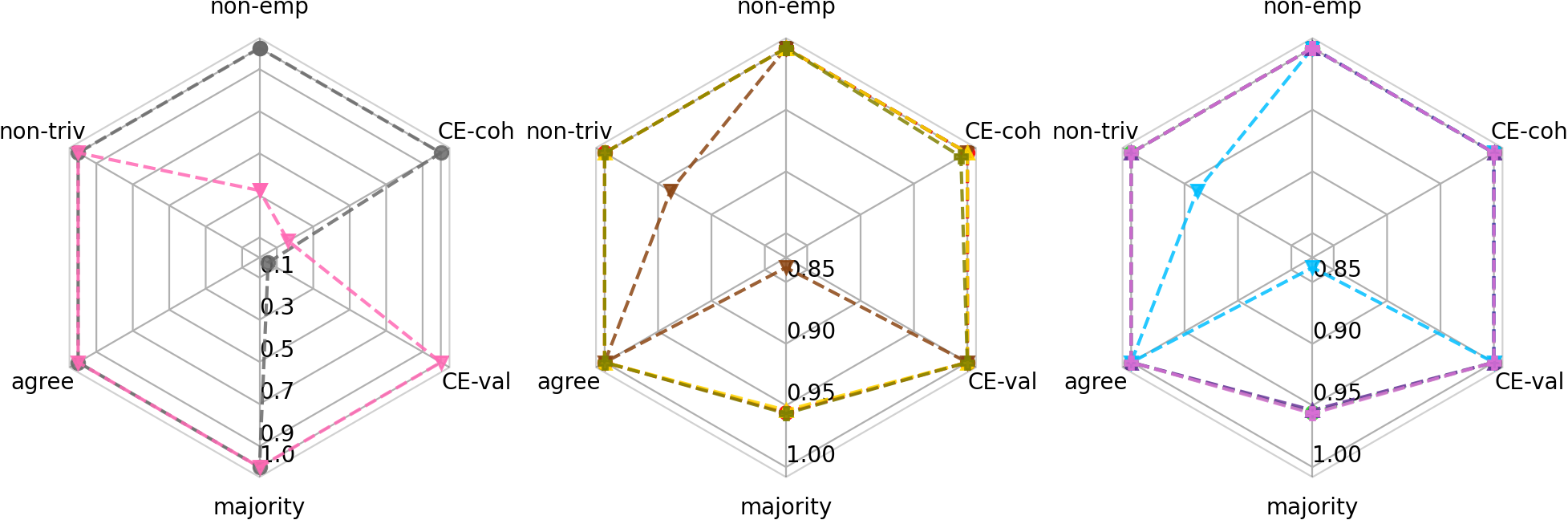}
         \caption{credit dataset, $|\Models|=30$}
         \label{fig:vis3_credit}
     \end{subfigure}
        \caption{\JJ{Satisfaction of desirable properties on credit dataset. The three subplots show results respectively for the two baselines, d-preferred argumentative ensembling, and s-preferred argumentative ensembling.}}
        \label{fig:property_satisfaction_credit}
\end{figure}

\JJ{Property satisfaction rates for the three datasets are visualised in Figures \ref{fig:property_satisfaction_heloc} to \ref{fig:property_satisfaction_credit}. As characterised by our earlier theoretical analysis, both baselines, $S^n$ and $S^v$, have significant limitations in terms of satisfying important properties (see Table \ref{table:props}) for them to be used in practice. $S^n$ demonstrates near-0 counterfactual validity, indicating that for most of the time, there are some CEs in the ensemble solution which are not valid on other models. $S^v$ produces valid CEs over models in the ensemble (100\% counterfactual validity), {but} they do not always exist, as shown by the low non-emptiness scores. In fact, for up to 70\% of test inputs, robust ensembling does not find any 
CEs ($S^v \cap \CFXs=\emptyset$), 
confirming its violation of non-emptiness.}
As $|\Models|$ increases, $S^v$ would require finding CEs which are valid for more models, and the number of CEs found would drop. 

\JJ{Argumentative ensembling, on the other hand, can always return non-empty ensembles while guaranteeing the validity of CEs. We observe that although theoretically d-preferred argumentative ensembling only satisfies model agreement and counterfactual validity, it practically always satisfies non-emptiness, and also achieves counterfactual coherence for over 95\% test points. Contrastively, all $S^{a,s}$ instantiations always satisfies model these four properties. Both d- and s-preferred argumentative ensembling satisfy non-triviality for most of the time, and more frequently so with larger $|\Models|$. They do not satisfy majority vote, although this has \AR{little} effect in terms of task accuracy. The four specific instantiations with different model preferences of $S^{a,d}$ and $S^{a,s}$ perform respectively very similarly to each other. $S^{a,\sigma}-A$ tends to have the lowest non-triviality and majority vote satisfaction, because if there exists a single most accurate model, it will be included in the ensemble (Proposition \ref{prop:dominant_model}), which may behave more differently than the other models. This is most obvious for the credit dataset.}


\subsection{\JJ{Results on Scalability}}
\label{ssec:scalability}

\begin{figure}[h!]
     \centering
     \begin{subfigure}{1\textwidth}
         \centering
         \includegraphics[width=\textwidth]{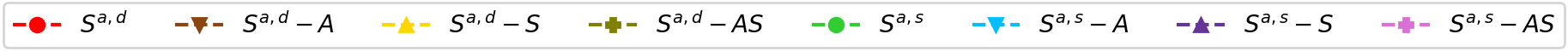}
     \end{subfigure}
     \begin{subfigure}{0.32\textwidth}
     \centering\includegraphics[width=\textwidth]{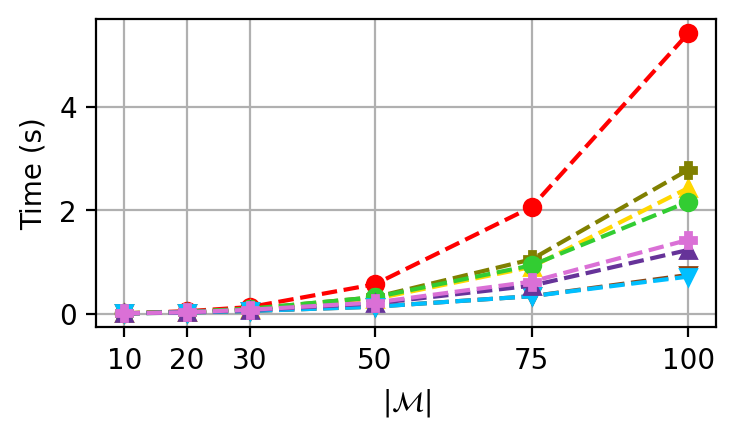}
         \caption{heloc dataset}
         \label{fig:time_heloc}
     \end{subfigure}
     \begin{subfigure}{0.32\textwidth}
         \centering
    \includegraphics[width=\textwidth]{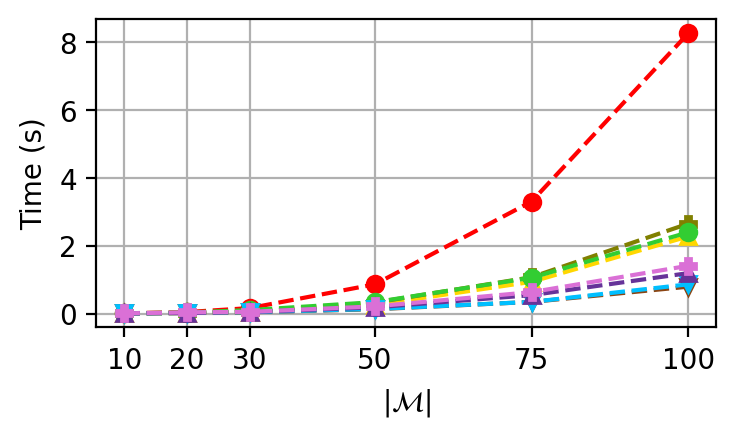}
         \caption{compas dataset}
         \label{fig:time_compas}
     \end{subfigure}
     \hfill
     \begin{subfigure}{0.32\textwidth}
         \centering
    \includegraphics[width=\textwidth]{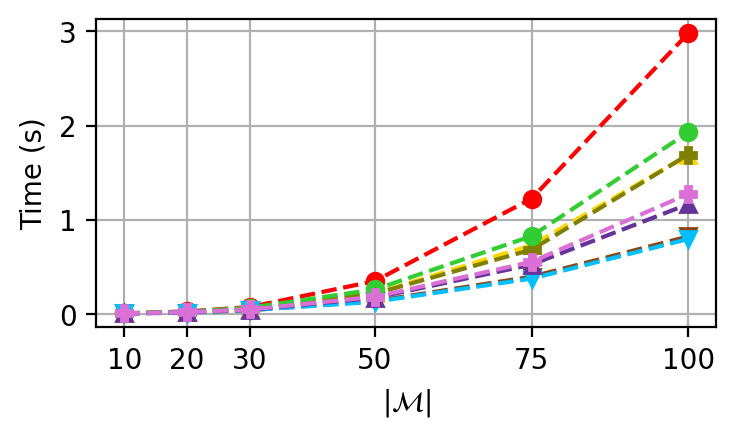}
         \caption{credit dataset}
         \label{fig:time_credit}
     \end{subfigure}
        \caption{\JJ{Average computation time of solving argumentative ensembling for each input.}}
        \label{fig:comp_time}
\end{figure}

\JJ{We further experiment with $|\Models|=50, 75, 100$ to illustrate the computation time required to solve the BAF in argumentative ensembling, although these numbers of models are less likely to occur in real-life applications. The computation time results are reported in Figure \ref{fig:comp_time}. It can be seen that $S^{a,s}$ instantiations are faster to solve than their d-preferred counterparts since they are more restrictive and thus have smaller search spaces. When experimenting with less than 30 models, the computation times are close to 0 for three datasets. Time differences between datasets for the same $|\Models|$, which are only governed by the number of conflicts in model predictions and CE validities, become more obvious as $|\Models|$ grows. Note that another part of the computation time of argumentative ensembling not reported in Figure \ref{fig:comp_time} is the time required to construct the BAF itself. This involves testing every CE on every model in $\Models$, and is $O(|\Models|^2)$ to the inference time of each model.}

\section{Conclusions and Future Work}
\label{sec:conclusion}


We
have presented a formal study of the problem of providing recourse under MM.
We defined several 
properties which are desirable in methods for solving this problem, highlighting deficiencies in 
extending conservatively the standard naive ensembling used for MM without recourse.
We have then introduced argumentative ensembling, a novel method for providing recourse under MM, which leverages computational argumentation to incorporate robustness guarantees and user preferences over models. 
We show, by means of a theoretical analysis, that argumentative ensembling hosts advantages over 
other methods, notably in non-emptiness of solutions and validity of CEs, notwithstanding its ability to handle user preferences.
This is, however, achieved by sacrificing the satisfaction of the property of majority vote
.
Our \emph{empirical} results, {however,} demonstrate that 
argumentative ensembling always finds valid CEs without compromising prediction accuracy, and shows the usefulness of specifying preferences over models.

\JJ{This work comes with some limitations. First, we assumed the availability of a valid CE for each model to instantiate the RAE problem. However, in practice, methods for finding CEs, especially the heuristic-based methods using gradient descent, may not always find valid CEs, limiting the range of CE generation methods to be used in our framework. It would be interesting to account for invalid CEs in the framework and leverage the power of argumentation to resolve inconsistencies and allow for more flexibility. Second, the desirable properties we proposed for RAE solutions are fundamental, binary properties. Additional metrics taking continuous values could potentially provide more fine-grained evaluations.}

This paper opens up several interesting directions for future work.
First, it would be interesting to examine whether considering attacks to or from \emph{sets} of arguments, rather than single arguments, as in \cite{Nielsen_06,Flouris_19,Dvorak_22,Dimopoulos_23}, 
may help in MM.
Further, extended AFs~\cite{Modgil_09} and value-based AFs \cite{value-based} may provide useful alternative ways to account for preferences.
We would also like to exploit the explanatory potential of argumentation to 
support explainable ensembling, e.g., using sub-graphs as in~\cite{Fan_14,Zeng_19}.
Moreover, in order to support experiments with a high number of models (beyond the 30 we considered),  large-scale argumentation solvers would be highly desirable. \JJ{In terms of CEs used, future research could also investigate how different types of CEs (e.g., plausible or not plausible, robust or not robust) interact with the property satisfaction in our context.}
Finally, it would be interesting to assess the effect which MM has on users' evaluations of CEs. 

\section*{Acknowledgements}
Jiang, Rago and Toni were partially funded by J.P. Morgan and by the Royal Academy of Engineering under the Research Chairs and Senior Research Fellowships scheme. 
Leofante was funded by Imperial College London through under the Imperial College Research Fellowship scheme. 
Rago and Toni were partially funded by the European Research Council (ERC) under the European Union’s Horizon 2020 research and innovation programme (grant agreement No. 101020934). 
Any views or opinions expressed herein are solely those of the authors listed.

\newpage

\bibliographystyle{elsarticle-num-names} 
\bibliography{refs}

\begin{thebibliography}{82}
\expandafter\ifx\csname natexlab\endcsname\relax\def\natexlab#1{#1}\fi
\providecommand{\url}[1]{\texttt{#1}}
\providecommand{\href}[2]{#2}
\providecommand{\path}[1]{#1}
\providecommand{\DOIprefix}{doi:}
\providecommand{\ArXivprefix}{arXiv:}
\providecommand{\URLprefix}{URL: }
\providecommand{\Pubmedprefix}{pmid:}
\providecommand{\doi}[1]{\href{http://dx.doi.org/#1}{\path{#1}}}
\providecommand{\Pubmed}[1]{\href{pmid:#1}{\path{#1}}}
\providecommand{\bibinfo}[2]{#2}
\ifx\xfnm\relax \def\xfnm[#1]{\unskip,\space#1}\fi
\bibitem[{Black et~al.(2022)Black, Raghavan, and Barocas}]{Black_22}
\bibinfo{author}{E.~Black}, \bibinfo{author}{M.~Raghavan}, \bibinfo{author}{S.~Barocas},
\newblock \bibinfo{title}{Model multiplicity: Opportunities, concerns, and solutions},
\newblock in: \bibinfo{booktitle}{FAccT 2022}, \bibinfo{year}{2022}, pp. \bibinfo{pages}{850--863}.
\bibitem[{Marx et~al.(2020)Marx, Calmon, and Ustun}]{Marx_20}
\bibinfo{author}{C.~T. Marx}, \bibinfo{author}{F.~P. Calmon}, \bibinfo{author}{B.~Ustun},
\newblock \bibinfo{title}{Predictive multiplicity in classification},
\newblock in: \bibinfo{booktitle}{{ICML} 2020}, \bibinfo{year}{2020}, pp. \bibinfo{pages}{6765--6774}.
\bibitem[{Breiman(2001)}]{breiman2001statistical}
\bibinfo{author}{L.~Breiman},
\newblock \bibinfo{title}{Statistical modeling: The two cultures (with comments and a rejoinder by the author)},
\newblock \bibinfo{journal}{Statistical science} \bibinfo{volume}{16} (\bibinfo{year}{2001}) \bibinfo{pages}{199--231}.
\bibitem[{Ganesh et~al.(2025)Ganesh, Taik, and Farnadi}]{ganesh2025curious}
\bibinfo{author}{P.~Ganesh}, \bibinfo{author}{A.~Taik}, \bibinfo{author}{G.~Farnadi},
\newblock \bibinfo{title}{The curious case of arbitrariness in machine learning},
\newblock \bibinfo{journal}{arXiv preprint arXiv:2501.14959}  (\bibinfo{year}{2025}).
\bibitem[{Semenova et~al.(2022)Semenova, Rudin, and Parr}]{semenova2022existence}
\bibinfo{author}{L.~Semenova}, \bibinfo{author}{C.~Rudin}, \bibinfo{author}{R.~Parr},
\newblock \bibinfo{title}{On the existence of simpler machine learning models},
\newblock in: \bibinfo{booktitle}{FAccT 2022}, \bibinfo{year}{2022}, pp. \bibinfo{pages}{1827--1858}.
\bibitem[{Hsu and Calmon(2022)}]{hsu2022rashomon}
\bibinfo{author}{H.~Hsu}, \bibinfo{author}{F.~P. Calmon},
\newblock \bibinfo{title}{Rashomon capacity: {A} metric for predictive multiplicity in classification},
\newblock in: \bibinfo{booktitle}{NeurIPS 2023}, \bibinfo{year}{2022}, pp. \bibinfo{pages}{28988--29000}.
\bibitem[{Watson{-}Daniels et~al.(2023)Watson{-}Daniels, Parkes, and Ustun}]{watson2023predictive}
\bibinfo{author}{J.~Watson{-}Daniels}, \bibinfo{author}{D.~C. Parkes}, \bibinfo{author}{B.~Ustun},
\newblock \bibinfo{title}{Predictive multiplicity in probabilistic classification},
\newblock in: \bibinfo{booktitle}{{AAAI} 2023}, \bibinfo{year}{2023}, pp. \bibinfo{pages}{10306--10314}.
\bibitem[{Black et~al.(2022)Black, Leino, and Fredrikson}]{black2022selective}
\bibinfo{author}{E.~Black}, \bibinfo{author}{K.~Leino}, \bibinfo{author}{M.~Fredrikson},
\newblock \bibinfo{title}{Selective ensembles for consistent predictions},
\newblock in: \bibinfo{booktitle}{{ICLR} 2022}, \bibinfo{year}{2022}.
\bibitem[{Guidotti(2022)}]{guidotti2022counterfactual}
\bibinfo{author}{R.~Guidotti},
\newblock \bibinfo{title}{Counterfactual explanations and how to find them: literature review and benchmarking},
\newblock \bibinfo{journal}{Data Mining and Knowledge Discovery}  (\bibinfo{year}{2022}) \bibinfo{pages}{1--55}.
\bibitem[{Karimi et~al.(2023)Karimi, Barthe, Sch{\"{o}}lkopf, and Valera}]{DBLP:journals/csur/KarimiBSV23}
\bibinfo{author}{A.~Karimi}, \bibinfo{author}{G.~Barthe}, \bibinfo{author}{B.~Sch{\"{o}}lkopf}, \bibinfo{author}{I.~Valera},
\newblock \bibinfo{title}{A survey of algorithmic recourse: Contrastive explanations and consequential recommendations},
\newblock \bibinfo{journal}{{ACM} Comput. Surv.} \bibinfo{volume}{55} (\bibinfo{year}{2023}) \bibinfo{pages}{95:1--95:29}.
\bibitem[{Jiang et~al.(2024)Jiang, Leofante, Rago, and Toni}]{DBLP:conf/ijcai/Jiangsurvey24}
\bibinfo{author}{J.~Jiang}, \bibinfo{author}{F.~Leofante}, \bibinfo{author}{A.~Rago}, \bibinfo{author}{F.~Toni},
\newblock \bibinfo{title}{Robust counterfactual explanations in machine learning: {A} survey},
\newblock in: \bibinfo{booktitle}{Proceedings of the Thirty-Third International Joint Conference on Artificial Intelligence, {IJCAI} 2024}, \bibinfo{year}{2024}, pp. \bibinfo{pages}{8086--8094}.
\bibitem[{Pawelczyk et~al.(2020)Pawelczyk, Broelemann, and Kasneci}]{pawelczyk2022uai}
\bibinfo{author}{M.~Pawelczyk}, \bibinfo{author}{K.~Broelemann}, \bibinfo{author}{G.~Kasneci},
\newblock \bibinfo{title}{On counterfactual explanations under predictive multiplicity},
\newblock in: \bibinfo{booktitle}{{UAI} 2020}, \bibinfo{year}{2020}, pp. \bibinfo{pages}{809--818}.
\bibitem[{Leofante et~al.(2023)Leofante, Botoeva, and Rajani}]{LeofanteBR23}
\bibinfo{author}{F.~Leofante}, \bibinfo{author}{E.~Botoeva}, \bibinfo{author}{V.~Rajani},
\newblock \bibinfo{title}{Counterfactual explanations and model multiplicity: a relational verification view},
\newblock in: \bibinfo{booktitle}{{KR} 2023}, \bibinfo{year}{2023}, pp. \bibinfo{pages}{763--768}.
\bibitem[{Coston et~al.(2021)Coston, Rambachan, and Chouldechova}]{coston2021characterizing}
\bibinfo{author}{A.~Coston}, \bibinfo{author}{A.~Rambachan}, \bibinfo{author}{A.~Chouldechova},
\newblock \bibinfo{title}{Characterizing fairness over the set of good models under selective labels},
\newblock in: \bibinfo{booktitle}{{ICML} 2021}, \bibinfo{year}{2021}, pp. \bibinfo{pages}{2144--2155}.
\bibitem[{Rudin(2019)}]{rudin2019stop}
\bibinfo{author}{C.~Rudin},
\newblock \bibinfo{title}{Stop explaining black box machine learning models for high stakes decisions and use interpretable models instead},
\newblock \bibinfo{journal}{Nat. Mach. Intell.} \bibinfo{volume}{1} (\bibinfo{year}{2019}) \bibinfo{pages}{206--215}.
\bibitem[{D'Amour et~al.(2022)D'Amour, Heller, Moldovan, Adlam, Alipanahi, Beutel, Chen, Deaton, Eisenstein, Hoffman et~al.}]{d2022underspecification}
\bibinfo{author}{A.~D'Amour}, \bibinfo{author}{K.~Heller}, \bibinfo{author}{D.~Moldovan}, \bibinfo{author}{B.~Adlam}, \bibinfo{author}{B.~Alipanahi}, \bibinfo{author}{A.~Beutel}, \bibinfo{author}{C.~Chen}, \bibinfo{author}{J.~Deaton}, \bibinfo{author}{J.~Eisenstein}, \bibinfo{author}{M.~D. Hoffman}, et~al.,
\newblock \bibinfo{title}{Underspecification presents challenges for credibility in modern machine learning},
\newblock \bibinfo{journal}{JMLR} \bibinfo{volume}{23} (\bibinfo{year}{2022}) \bibinfo{pages}{10237--10297}.
\bibitem[{Atkinson et~al.(2017)Atkinson, Baroni, Giacomin, Hunter, Prakken, Reed, Simari, Thimm, and Villata}]{AImagazine17}
\bibinfo{author}{K.~Atkinson}, \bibinfo{author}{P.~Baroni}, \bibinfo{author}{M.~Giacomin}, \bibinfo{author}{A.~Hunter}, \bibinfo{author}{H.~Prakken}, \bibinfo{author}{C.~Reed}, \bibinfo{author}{G.~R. Simari}, \bibinfo{author}{M.~Thimm}, \bibinfo{author}{S.~Villata},
\newblock \bibinfo{title}{Towards artificial argumentation},
\newblock \bibinfo{journal}{{AI} Magazine} \bibinfo{volume}{38} (\bibinfo{year}{2017}) \bibinfo{pages}{25--36}.
\bibitem[{han(2018)}]{handbook}
\bibinfo{title}{Handbook of Formal Argumentation}, \bibinfo{year}{2018}.
\bibitem[{Cayrol and Lagasquie{-}Schiex(2005)}]{Cayrol_05}
\bibinfo{author}{C.~Cayrol}, \bibinfo{author}{M.~Lagasquie{-}Schiex},
\newblock \bibinfo{title}{On the acceptability of arguments in bipolar argumentation frameworks},
\newblock in: \bibinfo{booktitle}{{ECSQARU} 2005}, \bibinfo{year}{2005}, pp. \bibinfo{pages}{378--389}.
\bibitem[{Jiang et~al.(2024)Jiang, Leofante, Rago, and Toni}]{DBLP:conf/atal/JiangL0T24}
\bibinfo{author}{J.~Jiang}, \bibinfo{author}{F.~Leofante}, \bibinfo{author}{A.~Rago}, \bibinfo{author}{F.~Toni},
\newblock \bibinfo{title}{Recourse under model multiplicity via argumentative ensembling},
\newblock in: \bibinfo{booktitle}{Proceedings of the 23rd International Conference on Autonomous Agents and Multiagent Systems, {AAMAS} 2024}, \bibinfo{year}{2024}, pp. \bibinfo{pages}{954--963}.
\bibitem[{Dung(1995)}]{Dung_95}
\bibinfo{author}{P.~M. Dung},
\newblock \bibinfo{title}{On the acceptability of arguments and its fundamental role in nonmonotonic reasoning, logic programming and n-person games},
\newblock \bibinfo{journal}{Artif. Intell.} \bibinfo{volume}{77} (\bibinfo{year}{1995}) \bibinfo{pages}{321--358}.
\bibitem[{Wick et~al.(2019)Wick, Panda, and Tristan}]{wick2019unlocking}
\bibinfo{author}{M.~L. Wick}, \bibinfo{author}{S.~Panda}, \bibinfo{author}{J.~Tristan},
\newblock \bibinfo{title}{Unlocking fairness: a trade-off revisited},
\newblock in: \bibinfo{booktitle}{NeurIPS 2019}, \bibinfo{year}{2019}, pp. \bibinfo{pages}{8780--8789}.
\bibitem[{Dutta et~al.(2020)Dutta, Wei, Yueksel, Chen, Liu, and Varshney}]{dutta2020there}
\bibinfo{author}{S.~Dutta}, \bibinfo{author}{D.~Wei}, \bibinfo{author}{H.~Yueksel}, \bibinfo{author}{P.~Chen}, \bibinfo{author}{S.~Liu}, \bibinfo{author}{K.~R. Varshney},
\newblock \bibinfo{title}{Is there a trade-off between fairness and accuracy? {A} perspective using mismatched hypothesis testing},
\newblock in: \bibinfo{booktitle}{{ICML} 2020}, \bibinfo{year}{2020}, pp. \bibinfo{pages}{2803--2813}.
\bibitem[{Rodolfa et~al.(2021)Rodolfa, Lamba, and Ghani}]{rodolfa2021empirical}
\bibinfo{author}{K.~T. Rodolfa}, \bibinfo{author}{H.~Lamba}, \bibinfo{author}{R.~Ghani},
\newblock \bibinfo{title}{Empirical observation of negligible fairness-accuracy trade-offs in machine learning for public policy},
\newblock \bibinfo{journal}{Nat. Mach. Intell.} \bibinfo{volume}{3} (\bibinfo{year}{2021}) \bibinfo{pages}{896--904}.
\bibitem[{Chen et~al.(2018)Chen, Lin, Rudin, Shaposhnik, Wang, and Wang}]{chen2018interpretable}
\bibinfo{author}{C.~Chen}, \bibinfo{author}{K.~Lin}, \bibinfo{author}{C.~Rudin}, \bibinfo{author}{Y.~Shaposhnik}, \bibinfo{author}{S.~Wang}, \bibinfo{author}{T.~Wang},
\newblock \bibinfo{title}{An interpretable model with globally consistent explanations for credit risk},
\newblock \bibinfo{journal}{CoRR} \bibinfo{volume}{abs/1811.12615} (\bibinfo{year}{2018}). \href{http://arxiv.org/abs/1811.12615}{{\tt arXiv:1811.12615}}.
\bibitem[{Dong and Rudin(2019)}]{dong2019variable}
\bibinfo{author}{J.~Dong}, \bibinfo{author}{C.~Rudin},
\newblock \bibinfo{title}{Variable importance clouds: {A} way to explore variable importance for the set of good models},
\newblock \bibinfo{journal}{CoRR} \bibinfo{volume}{abs/1901.03209} (\bibinfo{year}{2019}). \href{http://arxiv.org/abs/1901.03209}{{\tt arXiv:1901.03209}}.
\bibitem[{Fisher et~al.(2019)Fisher, Rudin, and Dominici}]{fisher2019all}
\bibinfo{author}{A.~Fisher}, \bibinfo{author}{C.~Rudin}, \bibinfo{author}{F.~Dominici},
\newblock \bibinfo{title}{All models are wrong, but many are useful: Learning a variable's importance by studying an entire class of prediction models simultaneously},
\newblock \bibinfo{journal}{J. Mach. Learn. Res.} \bibinfo{volume}{20} (\bibinfo{year}{2019}) \bibinfo{pages}{177:1--177:81}.
\bibitem[{Mehrer et~al.(2020)Mehrer, Spoerer, Kriegeskorte, and Kietzmann}]{mehrer2020individual}
\bibinfo{author}{J.~Mehrer}, \bibinfo{author}{C.~J. Spoerer}, \bibinfo{author}{N.~Kriegeskorte}, \bibinfo{author}{T.~C. Kietzmann},
\newblock \bibinfo{title}{Individual differences among deep neural network models},
\newblock \bibinfo{journal}{Nature communications} \bibinfo{volume}{11} (\bibinfo{year}{2020}) \bibinfo{pages}{5725}.
\bibitem[{Black et~al.(2022)Black, Wang, and Fredrikson}]{blackconsistent}
\bibinfo{author}{E.~Black}, \bibinfo{author}{Z.~Wang}, \bibinfo{author}{M.~Fredrikson},
\newblock \bibinfo{title}{Consistent counterfactuals for deep models},
\newblock in: \bibinfo{booktitle}{{ICLR} 2022}, \bibinfo{year}{2022}.
\bibitem[{Ley et~al.(2023)Ley, Tang, Nazari, Lin, Srinivas, and Lakkaraju}]{ley2023consistent}
\bibinfo{author}{D.~Ley}, \bibinfo{author}{L.~Tang}, \bibinfo{author}{M.~Nazari}, \bibinfo{author}{H.~Lin}, \bibinfo{author}{S.~Srinivas}, \bibinfo{author}{H.~Lakkaraju},
\newblock \bibinfo{title}{Consistent explanations in the face of model indeterminacy via ensembling},
\newblock \bibinfo{journal}{CoRR} \bibinfo{volume}{abs/2306.06193} (\bibinfo{year}{2023}). \href{http://arxiv.org/abs/2306.06193}{{\tt arXiv:2306.06193}}.
\bibitem[{Marx et~al.(2023)Marx, Park, Hasson, Wang, Ermon, and Huan}]{marx2023but}
\bibinfo{author}{C.~Marx}, \bibinfo{author}{Y.~Park}, \bibinfo{author}{H.~Hasson}, \bibinfo{author}{Y.~Wang}, \bibinfo{author}{S.~Ermon}, \bibinfo{author}{L.~Huan},
\newblock \bibinfo{title}{But are you sure? an uncertainty-aware perspective on explainable {AI}},
\newblock in: \bibinfo{booktitle}{AISTATS 2023}, \bibinfo{year}{2023}, pp. \bibinfo{pages}{7375--7391}.
\bibitem[{Xin et~al.(2022)Xin, Zhong, Chen, Takagi, Seltzer, and Rudin}]{xin2022exploring}
\bibinfo{author}{R.~Xin}, \bibinfo{author}{C.~Zhong}, \bibinfo{author}{Z.~Chen}, \bibinfo{author}{T.~Takagi}, \bibinfo{author}{M.~I. Seltzer}, \bibinfo{author}{C.~Rudin},
\newblock \bibinfo{title}{Exploring the whole rashomon set of sparse decision trees},
\newblock in: \bibinfo{booktitle}{NeurIPS 2022}, \bibinfo{year}{2022}, pp. \bibinfo{pages}{14071--14084}.
\bibitem[{Zhong et~al.(2023)Zhong, Chen, Liu, Seltzer, and Rudin}]{DBLP:conf/nips/Zhong00SR23}
\bibinfo{author}{C.~Zhong}, \bibinfo{author}{Z.~Chen}, \bibinfo{author}{J.~Liu}, \bibinfo{author}{M.~I. Seltzer}, \bibinfo{author}{C.~Rudin},
\newblock \bibinfo{title}{Exploring and interacting with the set of good sparse generalized additive models},
\newblock in: \bibinfo{booktitle}{NeurIPS 2023}, \bibinfo{year}{2023}.
\bibitem[{Hsu et~al.(2024{\natexlab{a}})Hsu, Li, Hu, and Chen}]{hsudropout}
\bibinfo{author}{H.~Hsu}, \bibinfo{author}{G.~Li}, \bibinfo{author}{S.~Hu}, \bibinfo{author}{C.-F. Chen},
\newblock \bibinfo{title}{Dropout-based rashomon set exploration for efficient predictive multiplicity estimation},
\newblock in: \bibinfo{booktitle}{The Twelfth International Conference on Learning Representations}, \bibinfo{year}{2024}{\natexlab{a}}.
\bibitem[{Hsu et~al.(2024{\natexlab{b}})Hsu, Brugere, Sharma, L{\'{e}}cu{\'{e}}, and Chen}]{DBLP:conf/nips/HsuBSLC24}
\bibinfo{author}{H.~Hsu}, \bibinfo{author}{I.~Brugere}, \bibinfo{author}{S.~Sharma}, \bibinfo{author}{F.~L{\'{e}}cu{\'{e}}}, \bibinfo{author}{R.~Chen},
\newblock \bibinfo{title}{Rashomongb: Analyzing the rashomon effect and mitigating predictive multiplicity in gradient boosting},
\newblock in: \bibinfo{booktitle}{NeurIPS 2024}, \bibinfo{year}{2024}{\natexlab{b}}.
\bibitem[{Roth et~al.(2023)Roth, Tolbert, and Weinstein}]{roth2023reconciling}
\bibinfo{author}{A.~Roth}, \bibinfo{author}{A.~Tolbert}, \bibinfo{author}{S.~Weinstein},
\newblock \bibinfo{title}{Reconciling individual probability forecasts},
\newblock in: \bibinfo{booktitle}{FAccT 2023}, \bibinfo{year}{2023}, pp. \bibinfo{pages}{101--110}.
\bibitem[{Cavus and Biecek(2024)}]{DBLP:journals/corr/cavus24}
\bibinfo{author}{M.~Cavus}, \bibinfo{author}{P.~Biecek},
\newblock \bibinfo{title}{An experimental study on the rashomon effect of balancing methods in imbalanced classification},
\newblock \bibinfo{journal}{CoRR} \bibinfo{volume}{abs/2405.01557} (\bibinfo{year}{2024}). \href{http://arxiv.org/abs/2405.01557}{{\tt arXiv:2405.01557}}.
\bibitem[{Hamman et~al.(2024)Hamman, Dissanayake, Mishra, Lecue, and Dutta}]{hamman2024quantifying}
\bibinfo{author}{F.~Hamman}, \bibinfo{author}{P.~Dissanayake}, \bibinfo{author}{S.~Mishra}, \bibinfo{author}{F.~Lecue}, \bibinfo{author}{S.~Dutta},
\newblock \bibinfo{title}{Quantifying prediction consistency under model multiplicity in tabular llms},
\newblock \bibinfo{journal}{arXiv preprint arXiv:2407.04173}  (\bibinfo{year}{2024}).
\bibitem[{Potyka et~al.(2024)Potyka, Zhu, He, Kharlamov, and Staab}]{DBLP:conf/atal/PotykaZHKS24}
\bibinfo{author}{N.~Potyka}, \bibinfo{author}{Y.~Zhu}, \bibinfo{author}{Y.~He}, \bibinfo{author}{E.~Kharlamov}, \bibinfo{author}{S.~Staab},
\newblock \bibinfo{title}{Robust knowledge extraction from large language models using social choice theory},
\newblock in: \bibinfo{booktitle}{Proceedings of the 23rd International Conference on Autonomous Agents and Multiagent Systems, {AAMAS} 2024}, \bibinfo{year}{2024}, pp. \bibinfo{pages}{1593--1601}.
\bibitem[{Zhu et~al.(2024)Zhu, Potyka, Nayyeri, Xiong, He, Kharlamov, and Staab}]{DBLP:conf/emnlp/ZhuPNXHKS24}
\bibinfo{author}{Y.~Zhu}, \bibinfo{author}{N.~Potyka}, \bibinfo{author}{M.~Nayyeri}, \bibinfo{author}{B.~Xiong}, \bibinfo{author}{Y.~He}, \bibinfo{author}{E.~Kharlamov}, \bibinfo{author}{S.~Staab},
\newblock \bibinfo{title}{Predictive multiplicity of knowledge graph embeddings in link prediction},
\newblock in: \bibinfo{booktitle}{Findings of the Association for Computational Linguistics: {EMNLP} 2024}, \bibinfo{year}{2024}, pp. \bibinfo{pages}{334--354}.
\bibitem[{Tolomei et~al.(2017)Tolomei, Silvestri, Haines, and Lalmas}]{Tolomei_17}
\bibinfo{author}{G.~Tolomei}, \bibinfo{author}{F.~Silvestri}, \bibinfo{author}{A.~Haines}, \bibinfo{author}{M.~Lalmas},
\newblock \bibinfo{title}{Interpretable predictions of tree-based ensembles via actionable feature tweaking},
\newblock in: \bibinfo{booktitle}{KDD 2017}, \bibinfo{year}{2017}, pp. \bibinfo{pages}{465--474}.
\bibitem[{Wachter et~al.(2017)Wachter, Mittelstadt, and Russell}]{Wachter_17}
\bibinfo{author}{S.~Wachter}, \bibinfo{author}{B.~D. Mittelstadt}, \bibinfo{author}{C.~Russell},
\newblock \bibinfo{title}{Counterfactual explanations without opening the black box: Automated decisions and the {GDPR}},
\newblock \bibinfo{journal}{Harv. JL \& Tech.} \bibinfo{volume}{31} (\bibinfo{year}{2017}) \bibinfo{pages}{841}.
\bibitem[{Ustun et~al.(2019)Ustun, Spangher, and Liu}]{ustun2019actionable}
\bibinfo{author}{B.~Ustun}, \bibinfo{author}{A.~Spangher}, \bibinfo{author}{Y.~Liu},
\newblock \bibinfo{title}{Actionable recourse in linear classification},
\newblock in: \bibinfo{booktitle}{FAT 2019}, \bibinfo{year}{2019}, pp. \bibinfo{pages}{10--19}.
\bibitem[{Dhurandhar et~al.(2018)Dhurandhar, Chen, Luss, Tu, Ting, Shanmugam, and Das}]{dhurandhar2018explanations}
\bibinfo{author}{A.~Dhurandhar}, \bibinfo{author}{P.~Chen}, \bibinfo{author}{R.~Luss}, \bibinfo{author}{C.~Tu}, \bibinfo{author}{P.~Ting}, \bibinfo{author}{K.~Shanmugam}, \bibinfo{author}{P.~Das},
\newblock \bibinfo{title}{Explanations based on the missing: Towards contrastive explanations with pertinent negatives},
\newblock in: \bibinfo{booktitle}{NeurIPS 2018}, \bibinfo{year}{2018}, pp. \bibinfo{pages}{590--601}.
\bibitem[{Mothilal et~al.(2020)Mothilal, Sharma, and Tan}]{mothilal2020explaining}
\bibinfo{author}{R.~K. Mothilal}, \bibinfo{author}{A.~Sharma}, \bibinfo{author}{C.~Tan},
\newblock \bibinfo{title}{Explaining machine learning classifiers through diverse counterfactual explanations},
\newblock in: \bibinfo{booktitle}{FAT 2020}, \bibinfo{year}{2020}, pp. \bibinfo{pages}{607--617}.
\bibitem[{Looveren and Klaise(2021)}]{DBLP:conf/pkdd/LooverenK21}
\bibinfo{author}{A.~V. Looveren}, \bibinfo{author}{J.~Klaise},
\newblock \bibinfo{title}{Interpretable counterfactual explanations guided by prototypes},
\newblock in: \bibinfo{booktitle}{{ECML} {PKDD} 2021}, volume \bibinfo{volume}{12976} of \textit{\bibinfo{series}{Lecture Notes in Computer Science}}, \bibinfo{year}{2021}, pp. \bibinfo{pages}{650--665}.
\bibitem[{Kenny and Keane(2021)}]{DBLP:conf/aaai/KennyK21}
\bibinfo{author}{E.~M. Kenny}, \bibinfo{author}{M.~T. Keane},
\newblock \bibinfo{title}{On generating plausible counterfactual and semi-factual explanations for deep learning},
\newblock in: \bibinfo{booktitle}{AAAI 2021}, \bibinfo{year}{2021}, pp. \bibinfo{pages}{11575--11585}.
\bibitem[{Wu et~al.(2021)Wu, Ribeiro, Heer, and Weld}]{DBLP:conf/acl/WuRHW20}
\bibinfo{author}{T.~Wu}, \bibinfo{author}{M.~T. Ribeiro}, \bibinfo{author}{J.~Heer}, \bibinfo{author}{D.~S. Weld},
\newblock \bibinfo{title}{Polyjuice: Generating counterfactuals for explaining, evaluating, and improving models},
\newblock in: \bibinfo{booktitle}{ACL/IJCNLP 2021}, \bibinfo{year}{2021}, pp. \bibinfo{pages}{6707--6723}.
\bibitem[{Jiang et~al.(2025)Jiang, Bewley, Mishra, Lecue, and Veloso}]{jiang2025interpreting}
\bibinfo{author}{J.~Jiang}, \bibinfo{author}{T.~Bewley}, \bibinfo{author}{S.~Mishra}, \bibinfo{author}{F.~Lecue}, \bibinfo{author}{M.~Veloso},
\newblock \bibinfo{title}{Interpreting language reward models via contrastive explanations},
\newblock in: \bibinfo{booktitle}{The 13th International Conference on Learning Representations, ICLR}, \bibinfo{year}{2025}.
\bibitem[{Upadhyay et~al.(2021)Upadhyay, Joshi, and Lakkaraju}]{upadhyay2021towards}
\bibinfo{author}{S.~Upadhyay}, \bibinfo{author}{S.~Joshi}, \bibinfo{author}{H.~Lakkaraju},
\newblock \bibinfo{title}{Towards robust and reliable algorithmic recourse},
\newblock in: \bibinfo{booktitle}{NeurIPS 2021}, \bibinfo{year}{2021}, pp. \bibinfo{pages}{16926--16937}.
\bibitem[{Dutta et~al.(2022)Dutta, Long, Mishra, Tilli, and Magazzeni}]{pmlr-v162-dutta22a}
\bibinfo{author}{S.~Dutta}, \bibinfo{author}{J.~Long}, \bibinfo{author}{S.~Mishra}, \bibinfo{author}{C.~Tilli}, \bibinfo{author}{D.~Magazzeni},
\newblock \bibinfo{title}{Robust counterfactual explanations for tree-based ensembles},
\newblock in: \bibinfo{booktitle}{ICML 2022}, \bibinfo{year}{2022}, pp. \bibinfo{pages}{5742--5756}.
\bibitem[{Jiang et~al.(2023)Jiang, Leofante, Rago, and Toni}]{oursaaai23}
\bibinfo{author}{J.~Jiang}, \bibinfo{author}{F.~Leofante}, \bibinfo{author}{A.~Rago}, \bibinfo{author}{F.~Toni},
\newblock \bibinfo{title}{Formalising the robustness of counterfactual explanations for neural networks},
\newblock in: \bibinfo{booktitle}{{AAAI} 2023}, \bibinfo{year}{2023}, pp. \bibinfo{pages}{14901--14909}.
\bibitem[{Hamman et~al.(2023)Hamman, Noorani, Mishra, Magazzeni, and Dutta}]{hamman2023robust}
\bibinfo{author}{F.~Hamman}, \bibinfo{author}{E.~Noorani}, \bibinfo{author}{S.~Mishra}, \bibinfo{author}{D.~Magazzeni}, \bibinfo{author}{S.~Dutta},
\newblock \bibinfo{title}{Robust counterfactual explanations for neural networks with probabilistic guarantees},
\newblock in: \bibinfo{booktitle}{{ICML} 2023}, \bibinfo{year}{2023}, pp. \bibinfo{pages}{12351--12367}.
\bibitem[{Jiang et~al.(2024{\natexlab{a}})Jiang, Lan, Leofante, Rago, and Toni}]{jiang2023provably}
\bibinfo{author}{J.~Jiang}, \bibinfo{author}{J.~Lan}, \bibinfo{author}{F.~Leofante}, \bibinfo{author}{A.~Rago}, \bibinfo{author}{F.~Toni},
\newblock \bibinfo{title}{Provably robust and plausible counterfactual explanations for neural networks via robust optimisation},
\newblock in: \bibinfo{booktitle}{Proceedings of the 15th Asian Conference on Machine Learning, ACML}, volume \bibinfo{volume}{222} of \textit{\bibinfo{series}{PMLR}}, \bibinfo{year}{2024}{\natexlab{a}}, pp. \bibinfo{pages}{582--597}.
\bibitem[{Jiang et~al.(2024{\natexlab{b}})Jiang, Leofante, Rago, and Toni}]{DBLP:journals/ai/JiangLRT24}
\bibinfo{author}{J.~Jiang}, \bibinfo{author}{F.~Leofante}, \bibinfo{author}{A.~Rago}, \bibinfo{author}{F.~Toni},
\newblock \bibinfo{title}{Interval abstractions for robust counterfactual explanations},
\newblock \bibinfo{journal}{Artif. Intell.} \bibinfo{volume}{336} (\bibinfo{year}{2024}{\natexlab{b}}) \bibinfo{pages}{104218}.
\bibitem[{Mishra et~al.(2021)Mishra, Dutta, Long, and Magazzeni}]{mishra2021survey}
\bibinfo{author}{S.~Mishra}, \bibinfo{author}{S.~Dutta}, \bibinfo{author}{J.~Long}, \bibinfo{author}{D.~Magazzeni},
\newblock \bibinfo{title}{A survey on the robustness of feature importance and counterfactual explanations},
\newblock \bibinfo{journal}{CoRR} \bibinfo{volume}{abs/2111.00358} (\bibinfo{year}{2021}). \href{http://arxiv.org/abs/2111.00358}{{\tt arXiv:2111.00358}}.
\bibitem[{Leofante and Wicker(2025)}]{Leofante2025}
\bibinfo{author}{F.~Leofante}, \bibinfo{author}{M.~Wicker}, \bibinfo{title}{Robustness of Counterfactual Explanations}, \bibinfo{publisher}{Springer Nature Switzerland}, \bibinfo{address}{Cham}, \bibinfo{year}{2025}, pp. \bibinfo{pages}{17--40}. \URLprefix \url{https://doi.org/10.1007/978-3-031-89022-2_3}. \DOIprefix\doi{10.1007/978-3-031-89022-2_3}.
\bibitem[{Cabrio and Villata(2013)}]{Cabrio_13}
\bibinfo{author}{E.~Cabrio}, \bibinfo{author}{S.~Villata},
\newblock \bibinfo{title}{A natural language bipolar argumentation approach to support users in online debate interactions{\textdagger}},
\newblock \bibinfo{journal}{Argument Comput.} \bibinfo{volume}{4} (\bibinfo{year}{2013}) \bibinfo{pages}{209--230}.
\bibitem[{Cyras et~al.(2019)Cyras, Letsios, Misener, and Toni}]{Cyras_19}
\bibinfo{author}{K.~Cyras}, \bibinfo{author}{D.~Letsios}, \bibinfo{author}{R.~Misener}, \bibinfo{author}{F.~Toni},
\newblock \bibinfo{title}{Argumentation for explainable scheduling},
\newblock in: \bibinfo{booktitle}{AAAI 2019}, \bibinfo{year}{2019}, pp. \bibinfo{pages}{2752--2759}.
\bibitem[{Irwin et~al.(2022)Irwin, Rago, and Toni}]{Irwin_22}
\bibinfo{author}{B.~Irwin}, \bibinfo{author}{A.~Rago}, \bibinfo{author}{F.~Toni},
\newblock \bibinfo{title}{Forecasting argumentation frameworks},
\newblock in: \bibinfo{booktitle}{KR 2022}, \bibinfo{year}{2022}, pp. \bibinfo{pages}{533--543}.
\bibitem[{Cyras et~al.(2021)Cyras, Rago, Albini, Baroni, and Toni}]{Cyras_21}
\bibinfo{author}{K.~Cyras}, \bibinfo{author}{A.~Rago}, \bibinfo{author}{E.~Albini}, \bibinfo{author}{P.~Baroni}, \bibinfo{author}{F.~Toni},
\newblock \bibinfo{title}{Argumentative {XAI:} {A} survey},
\newblock in: \bibinfo{booktitle}{{IJCAI} 2021}, \bibinfo{year}{2021}, pp. \bibinfo{pages}{4392--4399}.
\bibitem[{Vassiliades et~al.(2021)Vassiliades, Bassiliades, and Patkos}]{Vassiliades_21}
\bibinfo{author}{A.~Vassiliades}, \bibinfo{author}{N.~Bassiliades}, \bibinfo{author}{T.~Patkos},
\newblock \bibinfo{title}{Argumentation and explainable artificial intelligence: a survey},
\newblock \bibinfo{journal}{Knowl. Eng. Rev.} \bibinfo{volume}{36} (\bibinfo{year}{2021}) \bibinfo{pages}{e5}.
\bibitem[{Guo et~al.(2023)Guo, Yu, Bai, Tang, Ruan, and Zhou}]{Guo_23}
\bibinfo{author}{Y.~Guo}, \bibinfo{author}{T.~Yu}, \bibinfo{author}{L.~Bai}, \bibinfo{author}{J.~Tang}, \bibinfo{author}{Y.~Ruan}, \bibinfo{author}{Y.~Zhou},
\newblock \bibinfo{title}{Argumentative explanation for deep learning: A survey},
\newblock in: \bibinfo{booktitle}{ICUS 2023}, \bibinfo{year}{2023}, pp. \bibinfo{pages}{1738--1743}.
\bibitem[{Potyka(2021)}]{Potyka_21}
\bibinfo{author}{N.~Potyka},
\newblock \bibinfo{title}{Interpreting neural networks as quantitative argumentation frameworks},
\newblock in: \bibinfo{booktitle}{AAAI 2021}, \bibinfo{year}{2021}, pp. \bibinfo{pages}{6463--6470}.
\bibitem[{Dejl et~al.(2021)Dejl, He, Mangal, Mohsin, Surdu, Voinea, Albini, Lertvittayakumjorn, Rago, and Toni}]{Dejl_21}
\bibinfo{author}{A.~Dejl}, \bibinfo{author}{C.~He}, \bibinfo{author}{P.~Mangal}, \bibinfo{author}{H.~Mohsin}, \bibinfo{author}{B.~Surdu}, \bibinfo{author}{E.~Voinea}, \bibinfo{author}{E.~Albini}, \bibinfo{author}{P.~Lertvittayakumjorn}, \bibinfo{author}{A.~Rago}, \bibinfo{author}{F.~Toni},
\newblock \bibinfo{title}{Argflow: {A} toolkit for deep argumentative explanations for neural networks},
\newblock in: \bibinfo{booktitle}{{AAMAS} 2021}, \bibinfo{year}{2021}, pp. \bibinfo{pages}{1761--1763}.
\bibitem[{Timmer et~al.(2015)Timmer, Meyer, Prakken, Renooij, and Verheij}]{Timmer_15}
\bibinfo{author}{S.~T. Timmer}, \bibinfo{author}{J.~C. Meyer}, \bibinfo{author}{H.~Prakken}, \bibinfo{author}{S.~Renooij}, \bibinfo{author}{B.~Verheij},
\newblock \bibinfo{title}{Explaining bayesian networks using argumentation},
\newblock in: \bibinfo{booktitle}{{ECSQARU} 2015}, \bibinfo{year}{2015}, pp. \bibinfo{pages}{83--92}.
\bibitem[{Potyka et~al.(2023)Potyka, Yin, and Toni}]{Potyka_23}
\bibinfo{author}{N.~Potyka}, \bibinfo{author}{X.~Yin}, \bibinfo{author}{F.~Toni},
\newblock \bibinfo{title}{Explaining random forests using bipolar argumentation and markov networks},
\newblock in: \bibinfo{booktitle}{AAAI 2023}, \bibinfo{year}{2023}, pp. \bibinfo{pages}{9453--9460}.
\bibitem[{Leofante et~al.(2024)Leofante, Ayoobi, Dejl, Freedman, Gorur, Jiang, Paulino{-}Passos, Rago, Rapberger, Russo, Yin, Zhang, and Toni}]{DBLP:conf/kr/LeofanteADFGJP024}
\bibinfo{author}{F.~Leofante}, \bibinfo{author}{H.~Ayoobi}, \bibinfo{author}{A.~Dejl}, \bibinfo{author}{G.~Freedman}, \bibinfo{author}{D.~Gorur}, \bibinfo{author}{J.~Jiang}, \bibinfo{author}{G.~Paulino{-}Passos}, \bibinfo{author}{A.~Rago}, \bibinfo{author}{A.~Rapberger}, \bibinfo{author}{F.~Russo}, \bibinfo{author}{X.~Yin}, \bibinfo{author}{D.~Zhang}, \bibinfo{author}{F.~Toni},
\newblock \bibinfo{title}{Contestable {AI} needs computational argumentation},
\newblock in: \bibinfo{booktitle}{{KR} 2024}, \bibinfo{year}{2024}.
\bibitem[{Abchiche{-}Mimouni et~al.(2023)Abchiche{-}Mimouni, Amgoud, and Zehraoui}]{Abchiche-Mimouni_23}
\bibinfo{author}{N.~Abchiche{-}Mimouni}, \bibinfo{author}{L.~Amgoud}, \bibinfo{author}{F.~Zehraoui},
\newblock \bibinfo{title}{Explainable ensemble classification model based on argumentation},
\newblock in: \bibinfo{booktitle}{{AAMAS} 2023}, \bibinfo{year}{2023}, pp. \bibinfo{pages}{2367--2369}.
\bibitem[{FICO(2018)}]{heloc}
\bibinfo{author}{FICO}, \bibinfo{title}{Explainable machine learning challenge}, \bibinfo{year}{2018}.
\bibitem[{Julia~Angwin and Kirchner(2016)}]{compas}
\bibinfo{author}{S.~M. Julia~Angwin, Jeff~Larson}, \bibinfo{author}{L.~Kirchner}, \bibinfo{title}{There’s software used across the country to predict future criminals. and it’s biased against blacks.}, \bibinfo{year}{2016}.
\bibitem[{Hofmann(1994)}]{Dua2019}
\bibinfo{author}{H.~Hofmann}, \bibinfo{title}{{Statlog (German Credit Data)}}, \bibinfo{howpublished}{UCI Machine Learning Repository}, \bibinfo{year}{1994}.
\bibitem[{Brughmans et~al.(2023)Brughmans, Leyman, and Martens}]{NiceNNCE}
\bibinfo{author}{D.~Brughmans}, \bibinfo{author}{P.~Leyman}, \bibinfo{author}{D.~Martens},
\newblock \bibinfo{title}{{NICE:} an algorithm for nearest instance counterfactual explanations},
\newblock \bibinfo{journal}{Data Mining and Knowledge Discovery}  (\bibinfo{year}{2023}) \bibinfo{pages}{1--39}.
\bibitem[{Jiang et~al.(2025)Jiang, Marzari, Purohit, and Leofante}]{jiang2025robustx}
\bibinfo{author}{J.~Jiang}, \bibinfo{author}{L.~Marzari}, \bibinfo{author}{A.~Purohit}, \bibinfo{author}{F.~Leofante},
\newblock \bibinfo{title}{Robustx: Robust counterfactual explanations made easy},
\newblock \bibinfo{journal}{arXiv preprint arXiv:2502.13751}  (\bibinfo{year}{2025}).
\bibitem[{Nielsen and Parsons(2006)}]{Nielsen_06}
\bibinfo{author}{S.~H. Nielsen}, \bibinfo{author}{S.~Parsons},
\newblock \bibinfo{title}{A generalization of dung's abstract framework for argumentation: Arguing with sets of attacking arguments},
\newblock in: \bibinfo{booktitle}{ArgMAS 2006}, \bibinfo{year}{2006}, pp. \bibinfo{pages}{54--73}.
\bibitem[{Flouris and Bikakis(2019)}]{Flouris_19}
\bibinfo{author}{G.~Flouris}, \bibinfo{author}{A.~Bikakis},
\newblock \bibinfo{title}{A comprehensive study of argumentation frameworks with sets of attacking arguments},
\newblock \bibinfo{journal}{Int. J. Approx. Reason.} \bibinfo{volume}{109} (\bibinfo{year}{2019}) \bibinfo{pages}{55--86}.
\bibitem[{Dvor{\'{a}}k et~al.(2022)Dvor{\'{a}}k, K{\"{o}}nig, Ulbricht, and Woltran}]{Dvorak_22}
\bibinfo{author}{W.~Dvor{\'{a}}k}, \bibinfo{author}{M.~K{\"{o}}nig}, \bibinfo{author}{M.~Ulbricht}, \bibinfo{author}{S.~Woltran},
\newblock \bibinfo{title}{Rediscovering argumentation principles utilizing collective attacks},
\newblock in: \bibinfo{booktitle}{{KR} 2022}, \bibinfo{year}{2022}, pp. \bibinfo{pages}{122--131}.
\bibitem[{Dimopoulos et~al.(2023)Dimopoulos, Dvor{\'{a}}k, K{\"{o}}nig, Rapberger, Ulbricht, and Woltran}]{Dimopoulos_23}
\bibinfo{author}{Y.~Dimopoulos}, \bibinfo{author}{W.~Dvor{\'{a}}k}, \bibinfo{author}{M.~K{\"{o}}nig}, \bibinfo{author}{A.~Rapberger}, \bibinfo{author}{M.~Ulbricht}, \bibinfo{author}{S.~Woltran},
\newblock \bibinfo{title}{Sets attacking sets in abstract argumentation},
\newblock in: \bibinfo{booktitle}{NMR 2023}, \bibinfo{year}{2023}, pp. \bibinfo{pages}{22--31}.
\bibitem[{Modgil(2009)}]{Modgil_09}
\bibinfo{author}{S.~Modgil},
\newblock \bibinfo{title}{Reasoning about preferences in argumentation frameworks},
\newblock \bibinfo{journal}{Artif. Intell.} \bibinfo{volume}{173} (\bibinfo{year}{2009}) \bibinfo{pages}{901--934}.
\bibitem[{Bench{-}Capon(2002)}]{value-based}
\bibinfo{author}{T.~J.~M. Bench{-}Capon},
\newblock \bibinfo{title}{Value-based argumentation frameworks},
\newblock in: \bibinfo{booktitle}{{NMR} 2002}, \bibinfo{year}{2002}, pp. \bibinfo{pages}{443--454}.
\bibitem[{Fan and Toni(2014)}]{Fan_14}
\bibinfo{author}{X.~Fan}, \bibinfo{author}{F.~Toni},
\newblock \bibinfo{title}{On computing explanations in abstract argumentation},
\newblock in: \bibinfo{booktitle}{{ECAI} 2014}, \bibinfo{year}{2014}, pp. \bibinfo{pages}{1005--1006}.
\bibitem[{Zeng et~al.(2019)Zeng, Miao, Leung, Shen, and Chin}]{Zeng_19}
\bibinfo{author}{Z.~Zeng}, \bibinfo{author}{C.~Miao}, \bibinfo{author}{C.~Leung}, \bibinfo{author}{Z.~Shen}, \bibinfo{author}{J.~J. Chin},
\newblock \bibinfo{title}{Computing argumentative explanations in bipolar argumentation frameworks},
\newblock in: \bibinfo{booktitle}{AAAI 2019}, \bibinfo{year}{2019}, pp. \bibinfo{pages}{10079--10080}.

\end{thebibliography}

\end{document}